\documentclass[11pt]{article}
\usepackage{fullpage}
\usepackage[hyphens]{url}
\usepackage{hyperref}
\hypersetup{
breaklinks=true,
    colorlinks=true, 
    linkcolor=black, 
    citecolor=black, 
    filecolor=black,
    urlcolor=black 
}
\usepackage[utf8]{inputenc} 
\usepackage[T1]{fontenc}    
\usepackage{hyperref}       
\usepackage{url}            
\usepackage{booktabs}       
\usepackage{amsfonts}       
\usepackage{nicefrac}       
\usepackage{microtype}      
\usepackage{xcolor}         
\usepackage{pifont}
\usepackage{makecell}

\usepackage{float}

\usepackage{graphicx} 
\usepackage{caption}
\usepackage{wrapfig}

\usepackage{mathrsfs}
\usepackage{amsmath}
\usepackage{amsthm}
\usepackage{amssymb}
\usepackage{tablefootnote}
\usepackage{multirow}
\usepackage{enumerate}
\usepackage{color}
\usepackage{xcolor}

\usepackage[numbers]{natbib}

\usepackage{algorithm}
\usepackage{algpseudocode}
\allowdisplaybreaks[4]
\usepackage{bm,todonotes}
\allowdisplaybreaks

\newtheorem{thm}{Theorem}
\newtheorem{lem}{Lemma}[section]

\newtheorem{defn}{Definition}

\newcounter{subassumption}[asu]

\makeatletter
\renewcommand{\p@subassumption}{\theasu}
\makeatother

\newtheoremstyle{remarkstyle}
  {}                    
  {}                    
  {\normalfont}         
  {}                    
  {\itshape}            
  {.}                   
  { }                   
  {}                    
\theoremstyle{remarkstyle}

\usepackage{amsmath,amsfonts,bm}

\def\1{\bm{1}}

\DeclareMathAlphabet{\mathsfit}{\encodingdefault}{\sfdefault}{m}{sl}
\SetMathAlphabet{\mathsfit}{bold}{\encodingdefault}{\sfdefault}{bx}{n}

\renewcommand{\xi}{\zeta}

\def\0{{\bf 0}}
\def\1{{\bf 1}}

\def\PB{{\mathbb P}}

\def\KnowSum{{\textsc{KnowSum}}}

\title{
Evaluating the Unseen Capabilities: How Many Theorems Do LLMs Know?}
  
\author{
{Xiang Li\thanks{University of Pennsylvania; Email: \texttt{lx10077@upenn.edu}. } } 
\and
{Jiayi Xin\thanks{University of Pennsylvania; Email: \texttt{jiayixin@seas.upenn.edu}. }} 
\and
{Qi Long\thanks{University of Pennsylvania; Email: \texttt{qlong@upenn.edu}. }}
\and
{Weijie J. Su\thanks{University of Pennsylvania; Email: \texttt{suw@wharton.upenn.edu}. }}
}

\date{June 1, 2025}

\begin{document}

\maketitle

\begin{abstract}

Accurate evaluation of large language models (LLMs) is crucial for understanding their capabilities and guiding their development. However, current evaluations often inconsistently reflect the actual capacities of these models. In this paper, we demonstrate that one of many contributing factors to this \textit{evaluation crisis} is the oversight of unseen knowledge---information encoded by LLMs but not directly observed or not yet observed during evaluations. We introduce \KnowSum, a statistical framework designed to provide a more comprehensive assessment by quantifying the unseen knowledge for a class of evaluation tasks. \KnowSum~estimates the unobserved portion by extrapolating from the appearance frequencies of observed knowledge instances. We demonstrate the effectiveness and utility of \KnowSum~across three critical applications: estimating total knowledge, evaluating information retrieval effectiveness, and measuring output diversity. Our experiments reveal that a substantial volume of knowledge is omitted when relying solely on observed LLM performance. Importantly, \KnowSum~yields significantly different comparative rankings for several common LLMs based on their internal knowledge. 

\end{abstract}


\section{Introduction}\label{sec:intro}

Large language models (LLMs) have emerged as a transformative technology in artificial intelligence, demonstrating remarkable capabilities across diverse domains \citep{bommasani2021opportunities,achiam2023gpt}. Originally developed for natural language tasks such as translation, summarization, and dialogue generation, LLMs now exhibit increasingly sophisticated forms of reasoning and problem-solving \citep{brown2020language,wei2022chain,bubeck2023sparks}. They can contribute to clinical and biomedical decision-making~\citep{singhal2023large}, tackle mathematical problems with verifier assistance~\citep{cobbe2021training}, advance research in chemistry and physics~\citep{taylor2022galactica}, and synthesize knowledge across multiple scientific disciplines~\citep{bubeck2023sparks}. As the scope of their application continues to expand, LLMs are poised to become integral tools in scientific research, significantly influencing how knowledge is generated and utilized.

This remarkable progress has been facilitated, in large part, by LLM evaluation \citep{donoho2024data,chiang2024chatbot,chang2024survey}. LLM evaluation refers to the process of assessing and understanding model effectiveness through benchmark datasets and standardized protocols, which inform practitioners and guide model development. While evaluation has been fundamental to machine learning and AI progress since its early days, it has become increasingly critical nowadays as AI research has evolved into a predominantly empirical science, with advancement driven largely by experimentation and iterative refinement~\citep{sutton2019bitter,silver2025welcome}. Consequently, comprehensive evaluation of LLM capabilities is essential for identifying promising research directions and optimizing resource allocation, particularly given the substantial computational and energy requirements of AI development \citep{raji2021ai,bowman2024eight}.

Nevertheless, the field currently confronts a crisis of evaluation \citep{khomenko2025too, singh2025leaderboard}, characterized by a growing disconnect between benchmark performance and genuine model generalization capacity. Multiple factors contribute to this crisis, including benchmark contamination~\citep{sainz2023nlp}, overfitting through repeated leaderboard submissions~\citep{singh2025leaderboard}, and narrow test-time optimization strategies~\citep{saroufim2025neurips,cohen2025forget,zhang2024careful}, all of which artificially inflate performance metrics. However, a fundamental statistical limitation underlying this evaluation crisis remains largely unaddressed. Current evaluation methodologies typically consider only observable outputs from benchmark datasets. Due to the stochastic nature of generation, LLMs may possess substantial knowledge that remains unexpressed under constrained query budgets~\citep{wu2025estimating}. 
For instance, when attempting to assess an LLM's knowledge of 
mathematical theorems by directly prompting it to enumerate such conditions, the model would predominantly list common theorems, while rare theorems---despite being within the model's knowledge repertoire---might not surface even after hundreds of queries (see Figure~\ref{fig:fig2-motivation}).\footnote{Direct queries such as ``how many theorems do you know?'' is equally unreliable, as such responses are often inconsistent or exaggerated~\citep{kale2025line}.}

The aim of this paper is to quantify the extent of latent knowledge possessed by LLMs beyond observable outputs and to develop robust statistical methodologies for this purpose. While estimating the unseen may initially appear intractable, we demonstrate that by conceptualizing LLM outputs as a sampling process---analogous to drawing samples from Poisson or multinomial distributions---we can apply statistically principled techniques to infer unexpressed knowledge from observed knowledge counts. This challenge aligns with classical statistical problems such as estimating unseen species in ecology~\citep{good1953population} and inferring the number of unrecorded words in linguistics~\citep{efron1976estimating}.

\begin{figure}[!t]
\centering
\includegraphics[width=1\columnwidth]{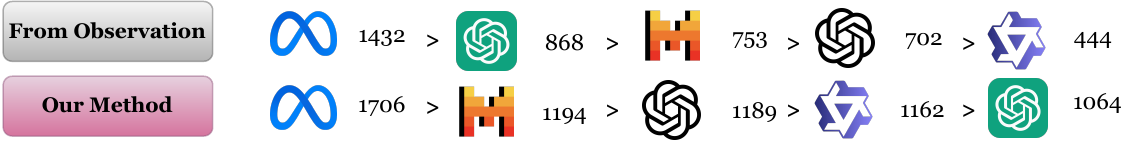}
\caption{
Estimating the number of unseen theorems changes LLM rankings based on the observed theorems alone. 
The evaluated models are: 
\raisebox{-0.2\ht\strutbox}{\includegraphics[height=0.35cm]{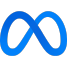}}~\texttt{LLaMA-V3-70B-instruct}~ \citep{grattafiori2024llama}, 
\raisebox{-0.2\ht\strutbox}{\includegraphics[height=0.35cm]{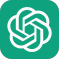}}~\texttt{ChatGPT-3.5-turbo-chat}~\citep{openai2023chatgpt35}, 
\raisebox{-0.2\ht\strutbox}{\includegraphics[height=0.35cm]{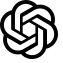}}~\texttt{ChatGPT-4o-chat}~\citep{hurst2024gpt},  
\raisebox{-0.2\ht\strutbox}{\includegraphics[height=0.35cm]{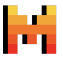}}~\texttt{Mistral-7B-instruct-V0.1}~\citep{jiang2023mistral7b}, and 
\raisebox{-0.2\ht\strutbox}{\includegraphics[height=0.35cm]{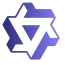}}~\texttt{Qwen2.5-7B-instruct}~\citep{yang2024qwen2}.
}
\label{fig:fig1-rank}
\end{figure}

\begin{figure}[!t]
\vspace{-10pt}
\centering
\includegraphics[width=1\columnwidth]{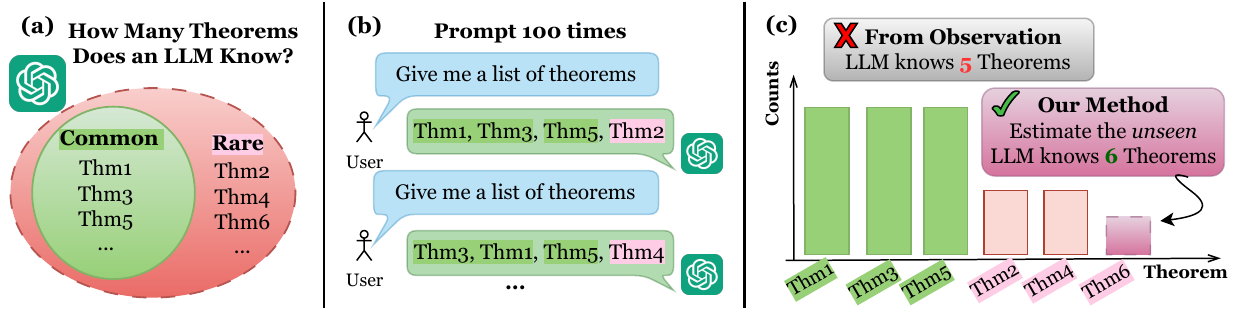}
\vspace{-20pt}
\caption{
\textbf{(a)} Examples of common and rare theorems from an LLM. 
\textbf{(b)} Obtain theorem names by repeatedly prompting the LLM. 
\textbf{(c)} Comparison between observed results and our proposed approach.
}
\label{fig:fig2-motivation}
\vspace{-10pt}
\end{figure}

To address this challenge, we introduce \KnowSum, a general-purpose statistical framework specifically designed to estimate unseen knowledge in LLMs. Rather than relying solely on observed outputs, \KnowSum~quantifies the internal knowledge of LLMs by extrapolating observed frequency distributions and employing the smoothed Good--Turing estimator \citep{orlitsky2016optimal}. A particularly attractive feature of \KnowSum~is that it incurs no additional computational cost. Moreover, the framework is statistically grounded, conceptually simple, and applicable to any evaluation task that can be formulated as knowledge counting. We demonstrate \KnowSum's versatility through three distinct applications: (1) estimating LLM knowledge of human diseases and mathematical theorems , (2) quantifying retrieval coverage in biomedical document retrieval and question answering, and (3) measuring semantic diversity in open-ended generation tasks. Our findings reveal two key insights: (i) LLMs typically express only 20-50\% of their estimated internal knowledge in our experiments, and (ii) accounting for unseen knowledge can markedly alter model rankings (see Figure~\ref{fig:fig1-rank}). By shifting focus from ``what the model outputs'' to ``what the model could potentially output,'' our work provides a novel perspective on model evaluation---one that more accurately reflects intrinsic capacity rather than merely observed performance.

\section{Related Work}
\label{sec:related-work}

\paragraph{LLM evaluation.}
A common approach to evaluating LLM capabilities is to pose a fixed set of questions and assess the model's responses based on correctness, much like a human exam \citep{chang2024survey,chiang2024chatbot}. 
Numerous benchmark datasets have been proposed for this purpose, spanning domains such as mathematics \citep{cobbe2021training}, healthcare \citep{singhal2023large}, science \citep{hendrycks2021measuring,cui2025curie}, and reasoning \citep{phan2025humanity}. However, these evaluations are limited to observable outputs on predefined questions and fail to capture knowledge that the model may possess but does not express. Our work complements this by estimating what a model could say but has not yet---which broadens the scope of standard evaluations.

In information retrieval, evaluation typically relies on ranking-based metrics (e.g., precision, recall, MRR \citep{Craswell2009}) or generation-based metrics (e.g., BLEU \citep{papineni2002bleu}, ROUGE \citep{lin2004rouge}). 
These approaches rely on query sets collected in advance, which may not fully capture the range of inputs a model can handle---thereby overlooking its potential to retrieve relevant information for unseen or underrepresented queries.
In contrast, our method estimates how much additional information a model could retrieve if more queries were drawn i.i.d. from the same distribution as the existing dataset, offering a broader view of retrieval coverage.
For diversity measurement, traditional metrics often rely on heuristics like perplexity \citep{Holtzman2020The} or type-token ratios \citep{templin1957certain}. Recent work \cite{wang2024can} applies hypothesis testing to assess whether LLMs exhibit human-like creativity. 
In contrast, we estimate how many semantically distinct outputs a model could generate for open-ended prompts, providing a scalable and interpretable measure of diversity beyond what is directly observed.

\vspace{-0.8em}
\paragraph{Estimating the unseen.}
Estimating the number of unseen from the observed is a fundamental problem arising across many domains \citep{efron1976estimating, chao2002estimating}. Despite differences in context, these problems often share a common structure: inferring properties of an underlying distribution from limited samples, where rare observations provide important information about what remains unobserved. 
 A classic and widely used approach is the Good--Turing estimator \citep{good1953population}, originally developed to estimate the number of unseen species, and later extended in contexts such as vocabulary diversity in computational linguistics \citep{efron1976estimating,gale1995good} and species richness in ecology \citep{chao2002estimating,chao2005new}. 
 In the context of LLMs, it has been used to quantify unseen facts in training data \citep{kalai2024calibrated} and to estimate memorized but unobserved facts in model outputs \citep{nasr2025scalable}.
 Although the Good--Turing estimator is theoretically grounded \citep{orlitsky2003always,orlitsky2015competitive,favaro2023near} and provides unbiased estimates, it suffers from extremely high variance. To address this, smoothed variants have been proposed \citep{efron1976estimating,orlitsky2016optimal} that better balance bias and variance, resulting in more stable and more accurate estimates. Therefore, we adopt such a smoothed estimator as a component of our pipeline.

\paragraph{Memorization, knowledge retrieval, and knowledge inference.}
Pre-trained language models are known to store a vast amount of knowledge \citep{petroni2019language,brown2020language,wang2020language}, and extensive research has investigated this capability. In many contexts, this knowledge resembles memorization---specifically, the ability to reproduce subtexts from the training dataset. Several works have focused on quantifying such memorization \citep{carlini2023quantifying,zhou2024quantifying,nasr2025scalable,shin2020autoprompt}, typically by crafting prompts that elicit memorized subtexts. 
Beyond memorization, several studies have examined how and when LLMs acquire factual knowledge during pretraining \citep{li2022pre,allen2023physics,chang2024large,prato2025large}, typically using external datasets as ground truth to score correctness. In contrast, we do not rely on predefined answers for grading or comparison. Instead, our focus is on estimating internal knowledge, that is, information encoded in the model that has not surfaced in sampled outputs.
Additionally, there is emerging work that infers properties of an LLM's internal knowledge. For example, \cite{gottesman2024estimating} predicts how knowledgeable a model is about a specific entity based on its internal activations before any output is produced. \cite{wu2025estimating} estimates the probability of a binary output property using importance sampling, particularly when the event is too rare to be observed directly.
While these methods similarly move beyond surface-level outputs, we focus on a different task---estimating the amount of internal knowledge.

\section{Method}
\label{sec:method}

\subsection{\KnowSum: Unseen Knowledge Estimation Pipeline}

\paragraph{Pipeline description.}
Algorithm~\ref{alg:detection} outlines our method which follows a five-step pipeline: generation, verification, clustering, prevalence estimation, and unseen estimation (see Figure~\ref{fig:fig2-flowchart}).
It begins by querying the model $n$ times to generate candidate knowledge items. A verifier is then used to filter out invalid or hallucinated responses using external sources such as search engines or domain-specific databases.
From each valid response, we extract the core knowledge using a decoder $C$, which typically removes stylistic variation to retain the essential content, for example, the name of a theorem. When a single response contains multiple answers, this step also splits them into individual items.
The resulting knowledge items are then clustered based on semantic similarity, with each cluster representing a distinct piece of knowledge. For example, ``Theorem of Pythagoras'' and ``Pythagorean theorem'' would be treated as equivalent.
We then construct a histogram $\{{n}_s\}_{s \ge 1}$ of item frequencies and apply the smoothed Good--Turing (SGT) estimator to predict the number of unseen items. Denoting by $\widehat{N}_{\mathrm{unseen}}(t)$ the number of items that do not appear in the observed $n$ queries but would appear in the subsequent $t \cdot n$ queries, the final estimate of total knowledge is given by $\widehat{N}_{\mathrm{tot}} = N_{\mathrm{seen}} + \widehat{N}_{\mathrm{unseen}}(t)$, where $N_{\mathrm{seen}} = \sum_{s \ge 1} {n}_s$ is the number of observed knowledge. 

\paragraph{Modularity.}
\KnowSum~is highly modular: once verification and clustering are specified for a given task, the rest of the pipeline applies out of the box. These two steps are the most task-dependent, as they reflect how knowledge is defined and validated in each domain. With them in place, the remaining steps remain unchanged, making the framework broadly applicable to any domain with discrete, countable knowledge. We detail the implementation choices in Section~\ref{sec:experiment}.

\begin{figure}[t]
\centering
\includegraphics[width=1\columnwidth]{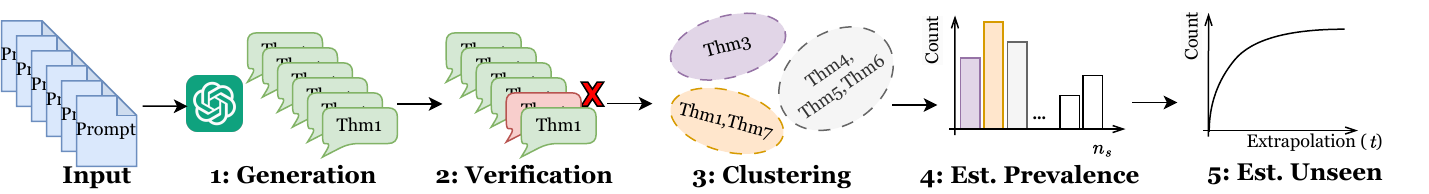}
\caption{Schematic illustration of the unseen knowledge estimation framework.}
\label{fig:fig2-flowchart}
\vspace{-2mm}
\end{figure}

\begin{algorithm}[t]
\caption{Knowledge Summation (\KnowSum)}
\label{alg:detection}
\begin{algorithmic}[1]
\Statex \textbf{Input:} Verifier $A$, knowledge decoder $C$, number of queries $n$, extrapolation factor $t$.

\State \textbf{Generation:} Independently query the target LLM $n$ times to produce responses $x_1, \ldots, x_n$.

\State \textbf{Verification:} Use the verifier $A$ to filter out invalid or hallucinated responses.

\State \textbf{Clustering:} Apply the decoder $C$ to extract the core knowledge $c_i = C(x_i)$ from each valid response, and group $\{c_i\}_{i=1}^n$ into $m$ semantically distinct clusters $\{k_j\}_{j=1}^m$. 

\State \textbf{Prevalence Estimation:} Construct the frequency histogram $\{{n}_s\}_{s \ge 1}$, where ${n}_s$ denotes the number of clusters that appear exactly $s$ times. Estimate the observed knowledge by $N_{\mathrm{seen}} = \sum_{s \ge 1} {n}_s$.

\State \textbf{Unseen Estimation:} For given $t$, use the smoothed Good--Turing estimator $\widehat{N}_{\mathrm{unseen}}(t)$ to compute the number of unseen knowledge items in $tn$ additional queries.

\Statex \textbf{Output:} Estimated total knowledge:  
$\widehat{N}_{\mathrm{tot}} := N_{\mathrm{seen}} + \widehat{N}_{\mathrm{unseen}}(t)$. 
\end{algorithmic}
\end{algorithm}

\subsection{Smoothed Good--Turing Estimator}

To quantify the amount of knowledge that remains unexpressed by an LLM, we adopt the smoothed Good--Turing estimator $\widehat{N}_{\mathrm{unseen}}(t)$ \citep{orlitsky2016optimal}. This estimator predicts the number of new knowledge items that would appear in the next $t \cdot n$ queries, based on the empirical prevalence counts $\{n_s\}_{s \ge 1}$ collected from the first $n$ queries, where $n_s$ denotes the number of items that appear exactly $s$ times.
The SGT estimator combines the first $k$ of these counts into a weighted sum:
\begin{equation}
\label{eq:SGT}
\widehat{N}_{\mathrm{unseen}}(t):= \sum_{s=1}^{k}  h_s\cdot {n}_s,
\quad \text{where} \quad
h_s= -(-t)^{s} \PB\left(\mathrm{Bin}\left(k, \frac{1}{t+1}\right) \geq s\right),
\end{equation}
with $k$ denoting a user-specified truncation parameter, and $\mathrm{Bin}(k, 1/(t+1))$ representing a binomial distribution with $k$ trials and success probability $1/(t+1)$. Initially known as the Efron–Thisted estimator \citep{efron1976estimating}, this estimator is provably near-optimal in terms of worst-case normalized mean squared error, up to subpolynomial factors \citep{orlitsky2016optimal}. 

The effectiveness of this estimator lies in a fundamental statistical insight: if each knowledge item emerges according to its own Poisson process, then rare items---those observed only once or twice---serve as indicators of a much larger set of unseen items with similarly low occurrence rates. Although the truly unseen items never appear in the data, their presence is inferred through these ``near misses''---items that almost went unobserved. These rare prevalence counts act as statistical fingerprints of the hidden mass, enabling reasonable extrapolation of how many new items would surface with continued querying. 

While the most direct signals come from singletons $n_1$ and doubletons $n_2$, incorporating the first $k$ prevalence counts $\{n_s\}_{s=1}^k$  allows the estimator to more robustly capture the contribution of low-frequency items and thus improves stability when extrapolating further into the unseen. The following result shows how $k$ controls the growth of the estimate $\widehat{N}_{\mathrm{unseen}}(t)$ relative to the number of previously seen items $N_{\mathrm{seen}}$.
For practical use, we select the optimal value of $k$ from a candidate set using a cross-validation procedure (see Section~\ref{sec:validation} for details).

\begin{thm}
\label{thm:relation}
If $t \ge k-1$, then the following holds :
\[
\widehat{N}_{\mathrm{unseen}}(t) \le  e^{\frac{kt}{t+1}} \cdot N_{\mathrm{seen}}.
\]
\end{thm}

This inequality also implies that the total predicted knowledge is at most a multiplicative factor of the observed knowledge. Motivated by this relationship, we introduce the seen knowledge ratio (SKR), which normalizes the estimate to yield a dimensionless quantity that reflects the proportion of exposed knowledge:

\begin{defn}[Seen Knowledge Ratio]
\label{def:SKR}
The seen knowledge ratio for a fixed $t$ is defined as
\[
\mathrm{SKR}(t) = \frac{N_{\mathrm{seen}}}{N_{\mathrm{seen}} + \widehat{N}_{\mathrm{unseen}}(t)}.
\]
\end{defn}
When both $t$ and $n$ are sufficiently large (e.g., $t = 10^2$, $n = 10^5$ in our experiments), the denominator $N_{\mathrm{seen}} + \widehat{N}_{\mathrm{unseen}}(t)$ serves as an estimate of the total knowledge encoded in the model. The SKR thus offers an interpretable summary of knowledge exposure: a value close to 1 indicates that most of the model's accessible knowledge has already surfaced, while a lower value suggests that much remains hidden. Moreover, since $\widehat{N}_{\mathrm{unseen}}(t)$ is non-decreasing in $t$, SKR naturally decreases as we consider further extrapolation into the model's unexpressed knowledge (see Figure~\ref{fig:effect-t-prompt}).

\section{Applications}
\label{sec:experiment}

In this section, we apply our \textsc{KnowSum} pipeline to three applications and evaluate the knowledge capabilities of several mainstream LLMs. At a high level, we focus on instruction-tuned or commercial models, as they are generally better at following prompts and adhering to task guidelines.
Specifically, we evaluate nine popular LLMs, listed in the first column of Table~\ref{tab:knowledge-counting}.
Due to space limits, additional experimental details are provided in Appendix~\ref{appen:experiment}.

\subsection{Application 1: Knowledge Estimation}

\paragraph{Setup.}
Our first and most basic application is to estimate the total amount of knowledge an LLM possesses. We focus on two domains: mathematical theorems and human diseases. The pipeline begins by querying the target LLM $N_{\mathrm{query}}$ times with a fixed prompt, each time requesting $N_{\mathrm{ans}}$ instances of domain-specific knowledge.
To validate the outputs, we use external databases and cluster the responses based on their external identifiers. Specifically, for mathematical theorems, we consult {Wikipedia}, {MathSciNet}, and {ProofWiki}, treating two theorems as identical if they are verified by the same webpage. For human diseases, we match each generated name against the {Human Disease Ontology} \citep{schriml2022human} and merge those that map to the same Disease Ontology Identifier (DOID).
We set the sampling temperature to $1$ to encourage diverse responses. To further promote diversity, we use $(N_{\mathrm{query}}, N_{\mathrm{ans}}) = (30{,}000, 20)$ for theorems and $(3{,}000, 50)$ for diseases.
Throughout all experiments, we fix the extrapolation factor at $t = 100$ with the number of total observations $n = N_{\mathrm{query}} \cdot N_{\mathrm{ans}}$. We selected the best truncation level $k \in \{6, 8, 10\}$ by cross-validation.
An sensitivity study on the effects of the prompt, $N_{\mathrm{query}}$, and $N_{\mathrm{ans}}$ is presented in Section~\ref{sec:analysis}.

\begin{table}[th]
\caption{Results of knowledge estimation (Application 1) across different LLMs. }
\label{tab:knowledge-counting}
\centering
\resizebox{\linewidth}{!}{%
\begin{tabular}{l|ccc|ccc||ccc|ccc}
\toprule
\multirow{2}{*}{\textbf{Model}} 
& \multicolumn{3}{c|}{\textbf{Theorem only}}
& \multicolumn{3}{c||}{\textbf{All math concepts}}
& \multicolumn{3}{c|}{\textbf{Anatomical disease}}
& \multicolumn{3}{c}{\textbf{Human diseases}} \\
& $N_{\mathrm{seen}}$ & $\widehat{N}_{\mathrm{tot}}$ & $\mathrm{SKR}$
& $N_{\mathrm{seen}}$ & $\widehat{N}_{\mathrm{tot}}$ & $\mathrm{SKR}$
& $N_{\mathrm{seen}}$ & $\widehat{N}_{\mathrm{tot}}$ & $\mathrm{SKR}$
& $N_{\mathrm{seen}}$ & $\widehat{N}_{\mathrm{tot}}$ & $\mathrm{SKR}$ \\
\midrule 
\ding{172} \texttt{ChatGPT-4o-chat} & 702 & 1189 & 0.59  & 974 & 2410 & 0.40 &  277 & 732 &  0.38 & 589  &  1096 & 0.54 \\
\ding{173} \texttt{ChatGPT-3.5-turbo-chat} &868 & 1064 & 0.82 & 1266 & 1703 & 0.74 & 268 & 278 & 0.96 & 523 & 706 & 0.74 \\
\ding{174} \texttt{LLaMA-V3-70B-instruct} & \textbf{1432} & \textbf{1706}& 0.84 & \textbf{2289} &\textbf{2645} & 0.87 & \textbf{875} & \textbf{3372}& 0.26 &  \textbf{1777} & \textbf{7564} & \textbf{0.23} \\
\ding{175} \texttt{LLaMA-V3-3B-instruct} &1035 & 1331 & 0.78 & 1717 & 2640 & 0.65 & 780 & 1375 & 0.57 & 1374 & 3002 & 0.46 \\
\ding{176} \texttt{Mistral-7B-instruct-V0.1} & 753 & 1194 & 0.63 & 1313 & 2481 & 0.53 & 489 & 1723 & 0.28 & 859 & 1276 & 0.67 \\
\ding{177} \texttt{Qwen2.5-7B-instruct} & 444 & 1162 & 0.38 & 663 & 1385 & 0.48 & 426 & 521 & 0.82 & 763 & 763 & 1.00\\
\ding{178} \texttt{Claude-3.7-Sonnet} &  120 & 201 & 0.60 & 147 & 293 & 0.50 & 115 & 462 & \textbf{0.25} & 213 & 686 & 0.31  \\
\ding{179} \texttt{DeepSeek-V3} &148 & 241 & 0.61 & 162 & 203 & 0.80 & 86 & 334 & 0.26 & 193 & 752 & 0.26 \\
\ding{180} \texttt{Gemini-1.5-flash} & 100 & 515 & \textbf{0.19} & 122 & 478 & \textbf{0.26} & 139 & 143 & 0.97 & 298 & 306 & 0.97\\
\bottomrule
\end{tabular}}
\vspace{-1em}
\end{table}

\paragraph{Two verification criteria.}
Our analysis estimates an LLM's total knowledge across two domains: human diseases and mathematical theorems. Table~\ref{tab:knowledge-counting} presents the counting results under two distinct validation criteria for each domain.
For \textbf{theorem counting}, the first two columns of Table~\ref{tab:knowledge-counting} detail our findings. We apply two validation criteria. The first, more restrictive setting, considers a response valid only if its name explicitly includes ``theorem''---even if the item is on Wikipedia, we mark it invalid without that specific word. The second, broader criterion accepts a response if it contains any of twelve related terms: ``theorem,'' ``lemma,'' ``law,'' ``principle,'' ``formula,'' ``criterion,'' ``identity,'' ``conjecture,'' ``rule,'' ``equation,'' ``postulate,'' or ``corollary.''  The term ``theorem'' itself accounts for roughly 10\% (2,558 entries) of the 24,762 identified mathematical concepts. 
The last two columns of Table~\ref{tab:knowledge-counting} show results for \textbf{human diseases}. Similarly,  we validate LLM-generated disease names using two criteria. Our initial, narrower criterion accepts responses only if they fall into the ``anatomical disease'' category (diseases affecting specific body parts), including recognized synonyms. Anatomical diseases make up about 51\% (6,036 entries) of the 11,872 human diseases in the Disease Ontology. Our second, broader criterion considers a response valid if it matches any human disease within the Disease Ontology, encompassing all its subcategories.
The distribution of these subcategories is provided in Figure~\ref{fig:subcategory_distribution} in the appendix.

\paragraph{A gap between observed and total knowledge.}
We observe that most {SKR} values are strictly less than 1, indicating that all evaluated LLMs possess unrevealed internal knowledge.
For theorem counting, \texttt{LLaMA-V3-70B-instruct} achieves the highest estimated total knowledge under both criteria. In the disease domain, it also reports the largest $\widehat{N}_{\mathrm{tot}}$ and the lowest SKR, suggesting broader coverage in biomedical knowledge.
\texttt{Gemini-1.5-flash} also exhibits a low SKR, likely due to its limited number of observed outputs, which implies a greater portion of its knowledge remains hidden.

\paragraph{Impact of unseen knowledge on model ranking.}
A notable finding is that accounting for unseen knowledge can meaningfully affect model comparison.  Take the ``theorem only'' setting as an example: based on the observed count $N_{\mathrm{seen}}$, \texttt{ChatGPT-3.5-turbo-chat} appears to outperform \texttt{ChatGPT-4o-chat}. However, when including unseen knowledge, \texttt{ChatGPT-4o-chat} is estimated to know more. This aligns with expectations, as the latter is trained on a larger corpus \citep{achiam2023gpt}.  A similar shift occurs under the broader validation criterion, where both \( N_{\mathrm{seen}} \) and \( \widehat{N}_{\mathrm{tot}} \) increase, but rankings still change.
Moreover, we observe the same reversal in the disease domain: although \texttt{DeepSeek-V3} and \texttt{Gemini-1.5-flash} appear comparable in observed counts, accounting for unseen knowledge shows that \texttt{DeepSeek-V3} has greater estimated coverage.

\begin{figure}[!t]
\centering
\includegraphics[width=\columnwidth]{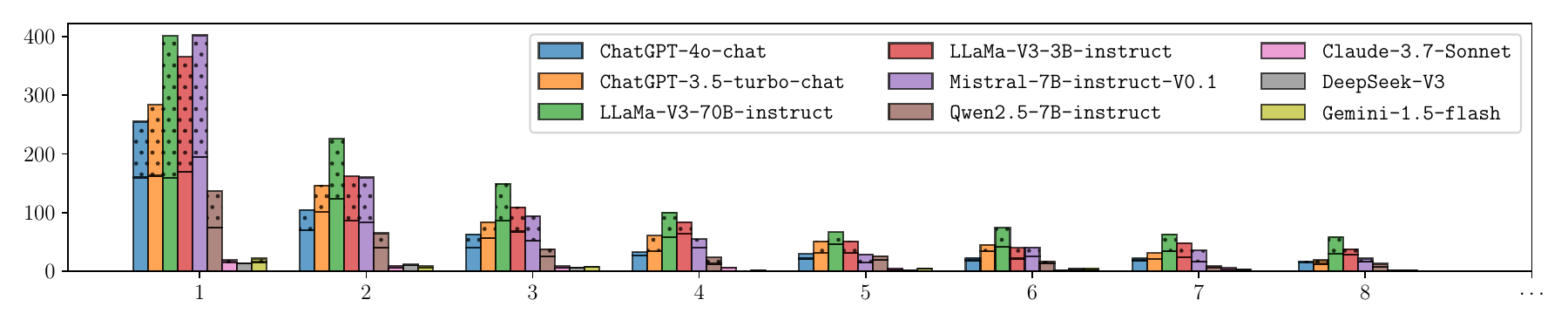}
\vspace{-0.6em}
\includegraphics[width=\columnwidth]{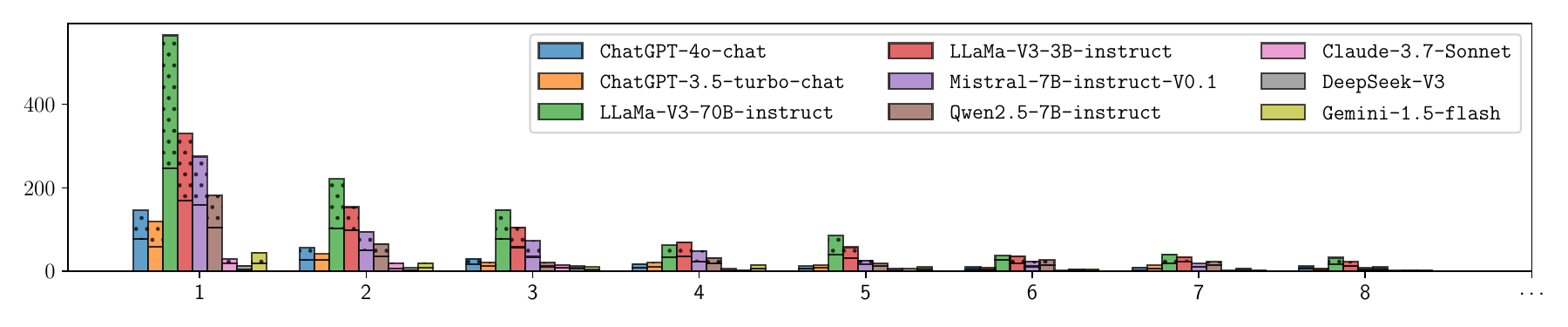}
\vspace{-2em}
\caption{Frequency histogram of theorems (top) and human diseases (bottom) found in LLM responses. The $x$-axis shows how many times a knowledge appears, and the $y$-axis shows how many distinct knowledges occur with that frequency. 
The shaded region shows additional counts due to relaxed validation criteria.
}
\label{fig:frequency}
\vspace{-1em}
\end{figure}

\paragraph{Why unseen knowledge alters model ranking.}  
A natural question is why accounting for unseen knowledge can change the ranking of LLMs. One might assume that a higher observed count \( N_{\mathrm{seen}} \) should imply a higher estimated total knowledge \( \widehat{N}_{\mathrm{tot}} \), but this is misleading---\( \widehat{N}_{\mathrm{tot}} \) is not determined by \( N_{\mathrm{seen}} \) alone.
Formally, \( \widehat{N}_{\mathrm{tot}} = \sum_{s \ge 1} n_s + \widehat{N}_{\mathrm{unseen}}(t) \), where \( n_s \) is the number of knowledge items seen exactly \( s \) times. While \( N_{\mathrm{seen}}=\sum_{s \ge 1} n_s  \) only counts distinct items, the SGT estimator depends on the prevalence distribution \( \{n_s\}_{s \ge 1} \), with coefficients that alternate in sign. Thus, even with the same \( N_{\mathrm{seen}} \), differences in prevalence (that is, how frequently items appear) can lead to different unseen estimates.
Figure~\ref{fig:frequency} illustrates this: all models show a decay in \( n_s \), but their scales and shapes differ. Notably, for theorems, \texttt{ChatGPT-3.5-turbo-chat} has uniformly higher \( n_s \) values than \texttt{ChatGPT-4o-chat}, yet receives a lower total knowledge estimate. This is because the sharper decay in the latter's prevalence suggests that it generates more rare or distinct items--leading to a higher unseen estimate and thus greater overall knowledge, despite having fewer observed items.

\begin{figure}[!t]
\centering
\includegraphics[width=0.495\textwidth]{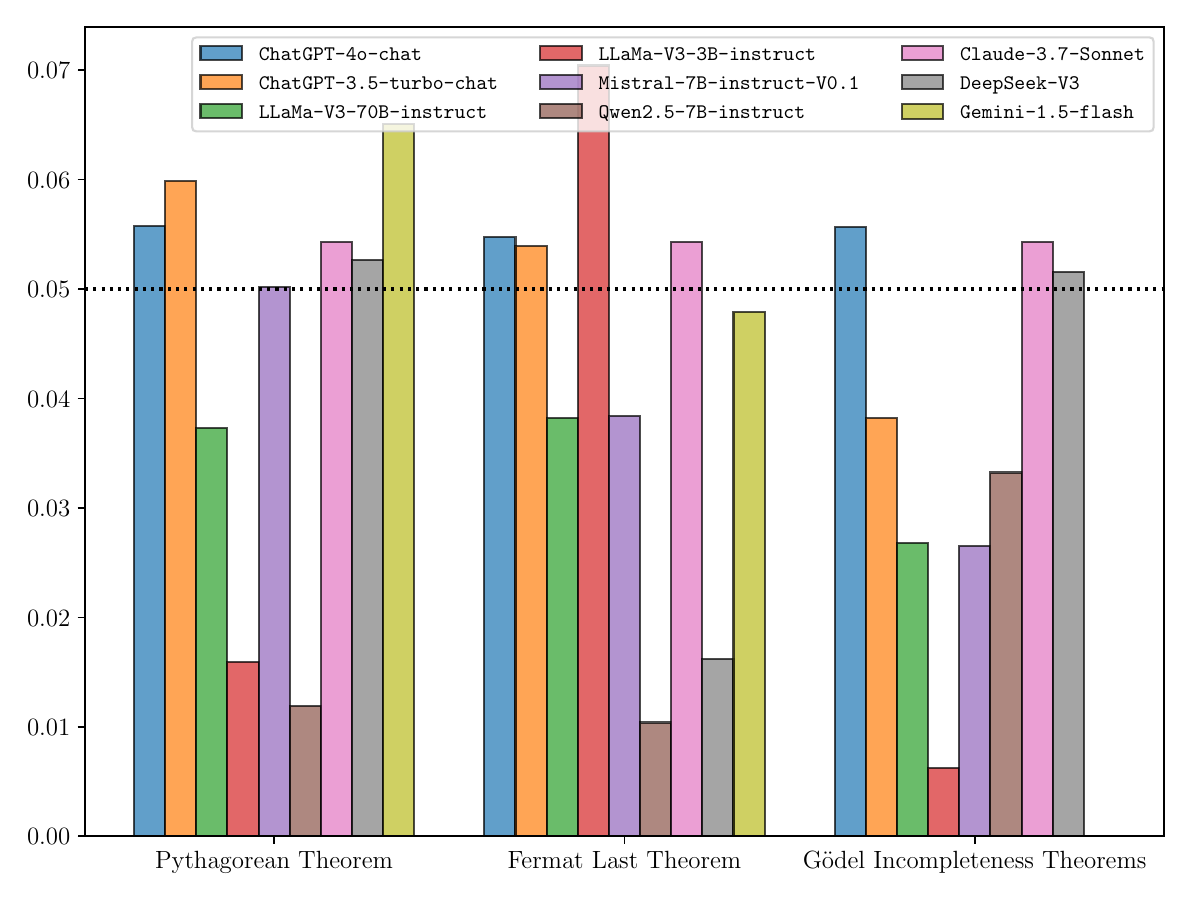}
\includegraphics[width=0.495\columnwidth]{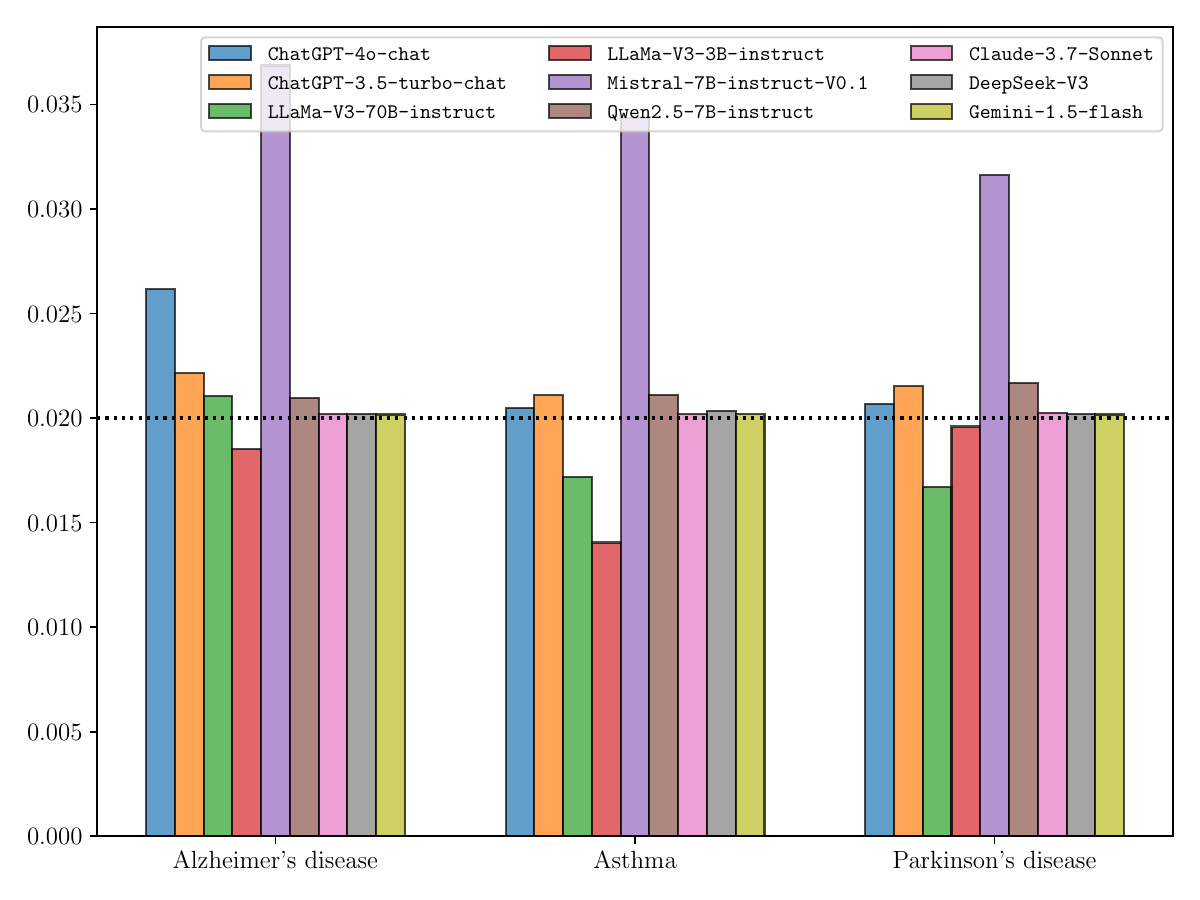}
\vspace{-15pt}  
\caption{Top three most frequently generated theorems (left) and human diseases (right). The black dotted line indicates the average relative frequency, defined by $1/N_{\mathrm{ans}}$, where $N_{\mathrm{ans}}$ is the number of knowledge items requested per query (20 for theorems and 50 for diseases).
}
\label{fig:top_theorem_count}
\vspace{-8pt}  
\end{figure}

\paragraph{Heavy-tailed prevalence.}
The prevalence $\{n_s\}_{s \ge 1}$ has a heavy-tailed distribution, that is, we observe nonzero counts $n_s \ge 1$ even for large values of $s$. 
This reflects LLMs' tendency to repeatedly generate a small number of highly common items, a pattern consistently observed across domains.
For theorem counting, as illustrated in Figure~\ref{fig:top_theorem_count} (left panel), LLMs frequently generate well-known examples like the Pythagorean Theorem, Fermat's Last Theorem, and Gödel's Incompleteness Theorems. 
Despite all models being queried 30,000 times (up to 20 responses per query), the number of valid theorems varies, affecting normalization. Nonetheless, the relative frequencies of these top responses---often exceeding 5\%---highlight their dominance and account for the heavy tail. 
Similarly, for human diseases, the figure on the right demonstrates LLMs repeatedly mention a small set of common human diseases (for example, Alzheimer's disease, Asthma, and Parkinson's disease). 
This motivates truncating the prevalence sequence in the SGT estimator. For instance, using $k = 6$ means that only $\{ n_s \}_{s=1}^6$ contribute to estimating $\widehat{N}_{\mathrm{seen}}(t)$, excluding high-frequency items. This truncation helps improve stability by reducing sensitivity to overrepresented responses.

\subsection{Application 2: Information Retrieval}

\paragraph{Setup.}
Our second application evaluates the information retrieval (IR) capabilities of LLMs using the \texttt{BioASQ-QA} Task 12b Test dataset \citep{krithara2023bioasq}, following the evaluation framework of \citep{zhou2024using}, focusing on two subtasks.
The first subtask is \textbf{document retrieval}, where the LLM is prompted to generate Boolean search queries to retrieve relevant documents from the PubMed database. These queries consist of multiple Medical Subject Headings (MeSH) keywords, combined using logical operators (AND, OR, NOT), and are submitted to a search engine to return candidate documents.
The second subtask is \textbf{question answering}, in which the LLM answers biomedical research questions---curated by domain experts---based on the retrieved documents.
In both subtasks, each question is associated with a set of ground-truth documents, and each document is annotated with a list of MeSH keywords. 
Instead of evaluating at the document level, we focus on MeSH keywords as retrieval-relevant knowledge units.
Each keyword reflects a specific biomedical concept that the model successfully identifies through retrieval. This offers a finer-grained view of what content the model can retrieve, beyond simply matching entire documents.
Indeed, traditional IR metrics---such as F1 score and ROUGE score 
 \citep{lin2004rouge}---assess retrieval and answer quality based on document relevance.
In contrast, our pipeline estimates how many additional relevant MeSH keywords an LLM could potentially retrieve---beyond those observed---if more questions were posed from a similar distribution and evaluated under the same validation criteria.

In our experiments, each LLM is evaluated on $N_{\text{query}} = 340$ biomedical questions using custom few-shot prompts (see Appendix~\ref{appen:prompts} for templates), with $N_{\text{ans}} = 1$ answer collected per question. The verification procedures differ across subtasks. In the document retrieval subtask, if the LLM retrieves a document that appears in the ground-truth set, all MeSH keywords associated with that document are counted as valid observed knowledge. In the question answering subtask, if the LLM's response is deemed correct, we include all MeSH keywords from the documents linked to that question. All MeSH keywords are normalized and clustered by the top-level concept of the MeSH hierarchy for consistent aggregation.

\begin{figure}[H]
  \centering
  \vspace{-5pt}  
  \includegraphics[width=0.4\textwidth]{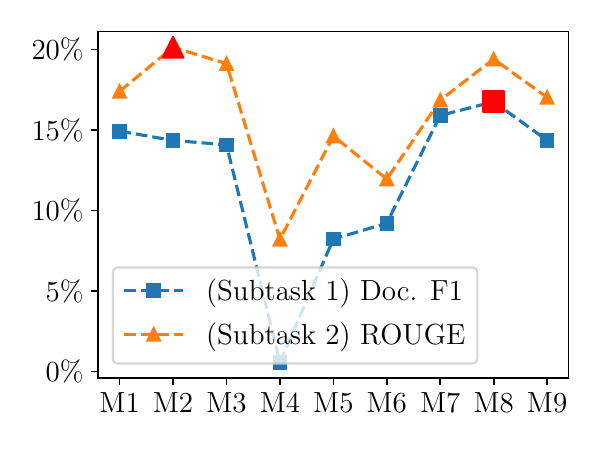}
  \vspace{-15pt}  
    \caption{Performance on selected traditional IR metrics. $M1$ denotes Model \ding{172} (\texttt{ChatGPT-4o-chat}), and similarly for the others (see Table \ref{tab:knowledge-counting} for the full list). The highest values are in red.}
  \label{fig:traditional-IR}
    \vspace{-5pt}  
\end{figure}

\paragraph{Results.}
Results for Application 2 are summarized in the left two columns of Table~\ref{tab:info-and-diversity}, with selected traditional IR metrics shown in Figure~\ref{fig:traditional-IR} and full results provided in Appendix~\ref{appen:traditional-IR}.
In Subtask 1, \texttt{DeepSeek-V3} achieves the highest observed count ($N_{\mathrm{seen}} = 2260$), consistent with its top document-level F1 score of 16.76\%. However, when incorporating estimated unseen knowledge, \texttt{ChatGPT-3.5-turbo-chat} surpasses all others with a total of $\widehat{N}_{\mathrm{tot}} = 10367$, significantly ahead of \texttt{DeepSeek-V3}'s 7750---demonstrating how our pipeline can reveal latent retrieval capacity not reflected in standard metrics.
In Subtask 2, \texttt{ChatGPT-4o-chat} leads in both observed and total MeSH terms, although its ROUGE score is slightly lower than the highest performer.
Overall, models exhibit SKR values around 25\% for Subtask 1 and 11\% for Subtask 2, both of which are lower than in Application 1.%

\begin{table}[!th]
\caption{Results of information retrieval (Application 2, left two columns) and diversity measures (Application 3, right two columns) across different LLMs.}
\label{tab:info-and-diversity}
\centering
\resizebox{\linewidth}{!}{%
\begin{tabular}{l|ccc|ccc||ccc|ccc}
\toprule
\multirow{2}{*}{\textbf{Model} }
& \multicolumn{3}{c|}{\textbf{Document Retrieval} }
& \multicolumn{3}{c||}{\textbf{Question Answering} }
& \multicolumn{3}{c|}{\textbf{LLM Applications} }
& \multicolumn{3}{c}{\textbf{Dream Jobs}} \\
& $N_{\mathrm{seen}}$ & $\widehat{N}_{\mathrm{tot}}$ &$\mathrm{SKR}$
& $N_{\mathrm{seen}}$ & $\widehat{N}_{\mathrm{tot}}$ &$\mathrm{SKR}$
& $N_{\mathrm{seen}}$ & $\widehat{N}_{\mathrm{tot}}$ &$\mathrm{SKR}$
& $N_{\mathrm{seen}}$ & $\widehat{N}_{\mathrm{tot}}$ &$\mathrm{SKR}$\\
\midrule 
\ding{172} \texttt{ChatGPT-4o-chat} & 2015 & 9676 & 0.21 & \textbf{2351} & \textbf{19965} &  0.12
& 165 &  714 &  0.23 & 409 &  1680 &  0.24 \\
\ding{173} \texttt{ChatGPT-3.5-turbo-chat} & 2190 & \textbf{10367} & 0.21 & 1850 & 15733 & 0.12 & 322 &  1339 &  0.24 & 131 &  560 &  0.23  \\
\ding{174} \texttt{LLaMA-V3-70B-instruct} & 1990 & 8488 & 0.23 & 1928 & 14270  & 0.14 & 437 &  1918 &  0.23 & 344 &  1487 &  0.23 \\
\ding{175} \texttt{LLaMA-V3-3B-instruct} & 79 & 396 &  \textbf{0.20} & 1653 & 14199 & 0.12 &   428 &  1926 &  0.22 & \textbf{770} &  \textbf{3386} &  0.23 \\
\ding{176} \texttt{Mistral-7B-instruct-v0.1} & 1364 & 5646 & 0.24 & 630 & 6596 & \textbf{0.10} &  {658} &  \textbf{3155} &  \textbf{0.21} & 233 &  1093 &  \textbf{0.21}  \\
\ding{177} \texttt{Qwen2.5-7B-instruct} & 1399 & 4853 & 0.28 & 1585 & 10710 & 0.15 &  421 &  1840 &  0.23 & 507 &  2094 &  0.24 \\
\ding{178} \texttt{Claude-3.7-Sonnet} & 2050 & 8831 & 0.23 & 2023 & 17230 & 0.12 &   \textbf{696} &  3013 &  0.23 & 133 &  543 &  0.24\\
\ding{179} \texttt{DeepSeek-V3} & \textbf{2260}  & 7750 & {0.30} & 2290 & 19744 & 0.12 &  17 &  48 &  0.35 & 7 &  10 &  0.7 \\
\ding{180} \texttt{Gemini-1.5-flash} & 2027 & 6616  & 0.31 & 2222 & 14898 & 0.15 &   21 &  37 &  0.57 & 3 &  10 &  0.3\\
\bottomrule
\end{tabular}}
\vspace{-1em}
\end{table}

\subsection{Application 3: Diversity Measure}

\paragraph{Setup.}
Our final application evaluates the diversity of different LLMs in response to open-ended prompts, such as ``Name a possible LLM application and explain why.'' Unlike the previous two applications, correctness verification is unnecessary, as any fluent and meaningful response is considered valid, which is typically satisfied by sufficiently large models.
To quantify diversity, we design two tasks. The first asks the LLM to describe a real or imagined application of LLMs. The second asks it to invent a job it might find appealing, if it could dream like a human. In both tasks, we instruct the models to generate short, simple responses with brief explanations. Each LLM is independently queried $N_{\mathrm{query}} = 1000$ times, collecting $N_{\mathrm{ans}} = 1$ response per query.
To cluster similar responses, we embed each response using OpenAI's \texttt{text-embedding-ada-002} model, which produces a 1536-dimensional semantic vector. Two responses are considered equivalent if the distance between their embeddings is smaller than a specified threshold. This clustering step is necessary, as LLMs may produce nearly identical applications or jobs but provide different reasons. 
To determine the threshold, we first compute the 10-nearest neighbor distances for each response within each LLM, resulting in one set of distances per model. We then aggregate these distances across all models and set the threshold to the $q$-quantile of the combined set.

\paragraph{Results.}
Results for diversity estimation at $q = 0.5$ are shown in the right two columns of Table~\ref{tab:info-and-diversity}. \texttt{Claude-3.7-Sonnet} produces the largest number of observed LLM applications, while \texttt{Mistral-7B-instruct-v0.1} has the highest estimated total knowledge. For the “dream jobs” task, \texttt{LLaMA-V3-3B-instruct} achieves the highest coverage in both observed count $N_{\mathrm{seen}}$ and total estimate $\widehat{N}_{\mathrm{tot}}$. In contrast, \texttt{DeepSeek-V3} and \texttt{Gemini-1.5-flash} produce the fewest outputs across both tasks, likely reflecting more deterministic generation behavior.
Most models exhibit an SKR near 23\%, indicating that only a small fraction of their internal knowledge is expressed through sampling. While \texttt{DeepSeek-V3} and \texttt{Gemini-1.5-flash} show higher SKRs (0.7 and 0.3, respectively), this likely results from limited generation diversity, leading to fewer unseen items.
This 23\% SKR is robust for different clustering thresholds $q \in \{0.2, 0.3, 0.5, 0.7\}$, with the top six models remaining within the 20–25\% range. This suggests a common tendency across LLMs: only a modest portion of their generative potential is typically realized.

\section{Validation and Sensitivity Analysis}
\label{sec:analysis}

\subsection{Cross-Validation of Unseen Knowledge Estimates}
\label{sec:validation}

To evaluate the accuracy of our pipeline, we perform held-out validation by splitting the $N_{\mathrm{query}}$ responses into two parts: the first $r_{\mathrm{obs}}$ fraction is treated as observed data, and we estimate the number of unique theorems that appear only in the remaining $(1 - r_{\mathrm{obs}})$ fraction. We repeat this procedure 100 times, each with a random shuffle of the full dataset so that the observed subset differs across runs. We report the average estimated and ground truth unseen counts in Figure~\ref{fig:sec-5-1-math_theorem_validation}, with error bars indicating the standard deviation over 100 repetitions.
We consider three values for $r_{\mathrm{obs}} \in \{1/2, 1/3, 1/4\}$, corresponding to a largest extrapolation factor of $t = 1, 2, 3$, respectively. Across all tested LLMs, the estimated unseen counts closely match the ground truth, implying that the SGT estimator yields highly accurate predictions for a wide range of extrapolation factors.

\begin{figure}[th]
\vspace{-5pt}  
\centering 
\includegraphics[width=\textwidth]{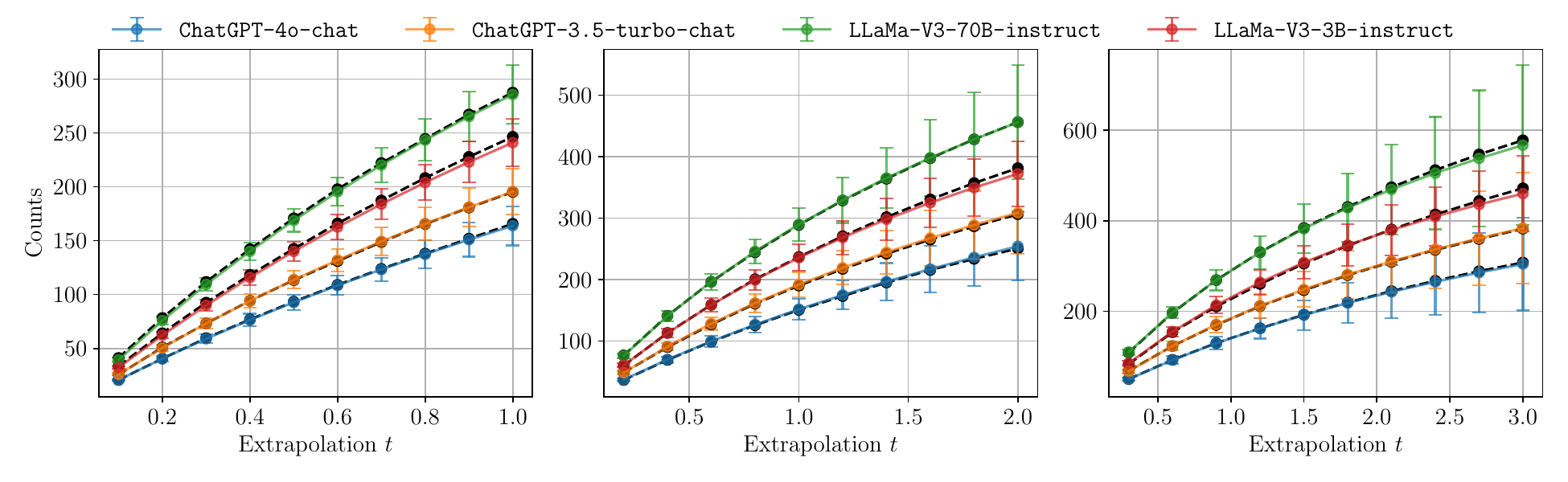}
\vspace{-20pt}  
\caption{SGT estimates (colored curves) versus ground truth (black dotted lines) for theorem estimation. From left to right, the observed fraction $r_{\mathrm{obs}}$ is $1/2$, $1/3$, and $1/4$. Counterpart results for human diseases are shown in Figure~\ref{fig:appendix-disease_validation} in the Appendix.}
\label{fig:sec-5-1-math_theorem_validation}
\vspace{-5pt}  
\end{figure}

\subsection{Sensitivity Analysis}

\paragraph{Effect of the extrapolation factor $t$.}
Next, we examine the effect of the extrapolation factor $t$. By definition, as $t$ increases, the estimated number of unseen items $\widehat{N}_{\mathrm{seen}}(t)$ should grow so that the SKR decreases. This trend is confirmed in the left two plots of Figure~\ref{fig:effect-t-prompt}, where we observe a monotonic increase in $\widehat{N}_{\mathrm{seen}}(t)$ and a corresponding decrease in SKR. Additionally, the variance of the estimate becomes larger as $t$ increases, reflecting greater uncertainty in longer-range extrapolation. In practice, we find that both quantities tend to saturate around $t=80$, with little gain beyond that point. This motivates our choice of $t = 100$ as a default in all experiments.

\begin{figure}[th]
\vspace{-5pt}  
\centering 
\includegraphics[width=0.7\textwidth]{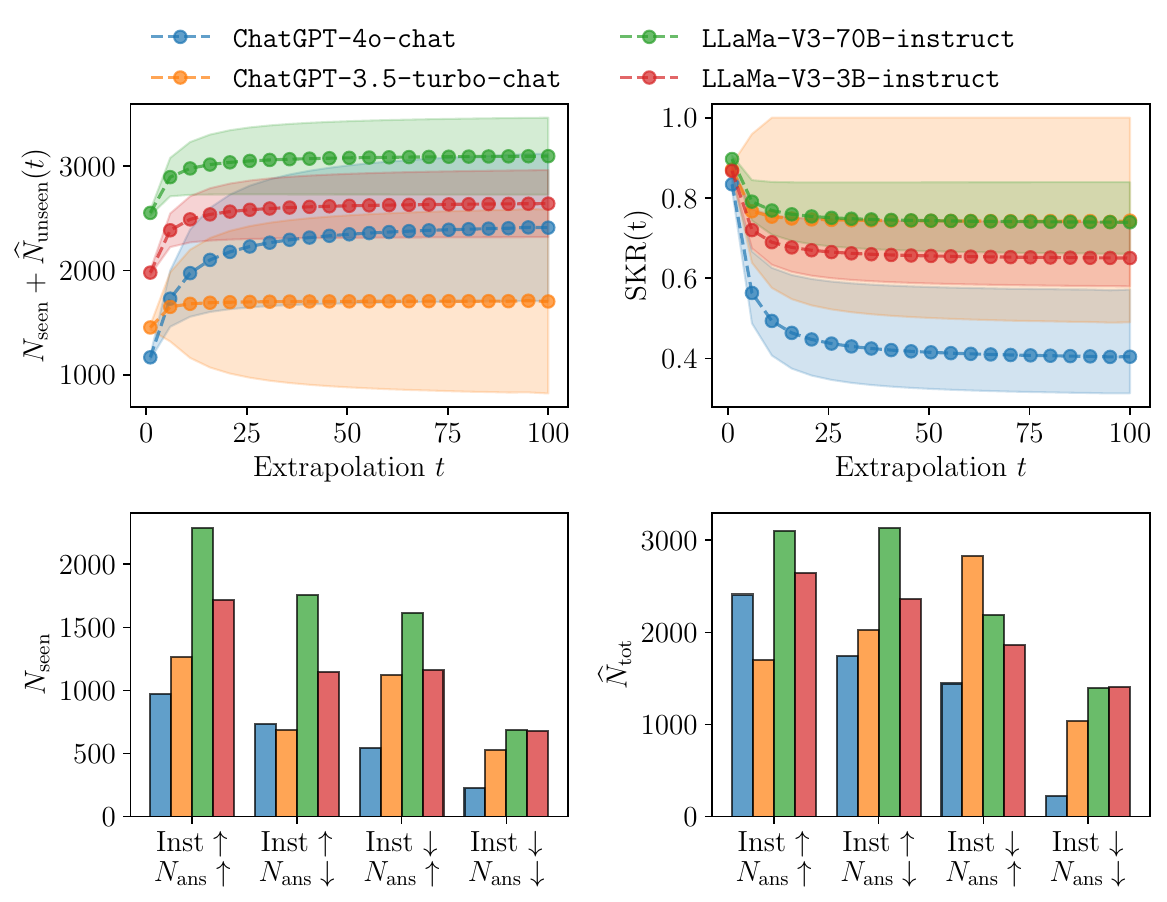}
\vspace{-10pt}  
\caption{Effects of the extrapolation factor $t$ (left two plots) and prompts (right two plots) for theorem
estimation. Counterpart results for human diseases are shown in Figure~\ref{fig:appendix-disease_effect-t-prompt} in the Appendix.
}
\label{fig:effect-t-prompt}
\vspace{-5pt}  
\end{figure}

\paragraph{Effect of the prompt.}
An important practical issue is how to design the prompt. We examine two key factors: the clarity of the instruction ($\mathrm{Inst}\uparrow$ or $\mathrm{Inst}\downarrow$) and the number of requested responses per query ($N_{\mathrm{ans}} \in \{10, 20\}$). We test four configurations that vary along these two dimensions. The results of theorem counting are shown in the right two plots of Figure~\ref{fig:effect-t-prompt}. Overall, prompts with clearer instructions and a higher number of requested responses lead to more diverse and informative outputs, increasing both $N_{\mathrm{seen}}$ and $\widehat{N}_{\mathrm{tot}}$ in our experiments, and thereby improving access to the model's internal knowledge. The set of used prompts is listed in Appendix \ref{appen:prompts}.

\paragraph{Effect of the truncation level $k$.}

The first plot in Figure~\ref{fig:effect-truncation-temp-sampling} shows how the estimated total knowledge $\widehat{N}_{\mathrm{tot}}$ varies with the truncation level $k$, while the second plot presents the normalized mean square error (MSE).
Overall, the estimator is relatively insensitive to the choice of $k$, suggesting that the prevalence histogram ${n_s}$ is generally stable and informative, which enables accurate extrapolation over a broad range of truncation levels.
Nonetheless, we recommend using the held-out validation procedure described in Section~\ref{sec:validation} to select $k$ adaptively for each setting. As illustrated in the second plot, most models perform best when $k=8$. The complete set of selected $k$ values for each LLM and application is provided in Appendix~\ref{appen:truncation-levels}.

\begin{figure}[ht]
\vspace{-8pt}  
\centering 
\includegraphics[width=0.7\textwidth]{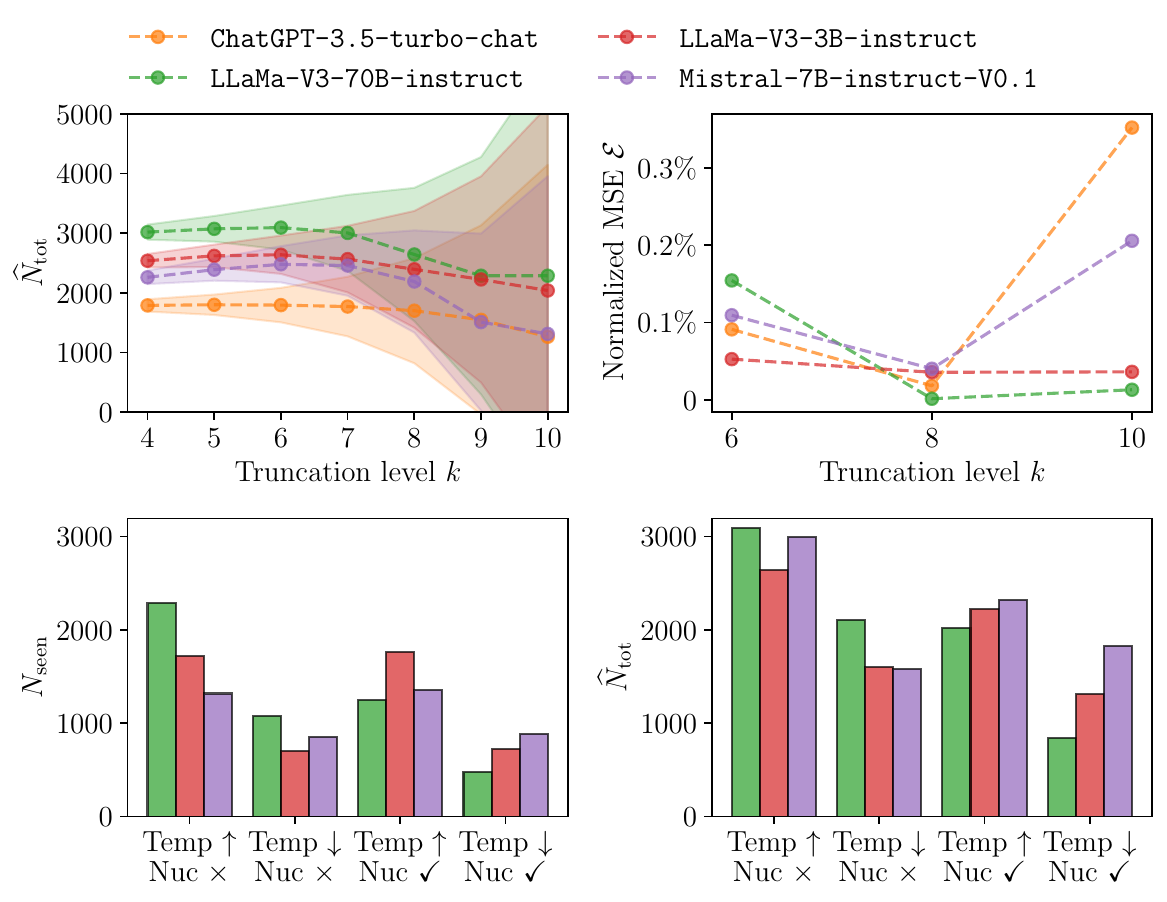}
\vspace{-10pt}  
\caption{Effects of the truncation level $k$ (left two plots) and sampling strategies (right two plots) for theorem
estimation. Counterpart results for human diseases are shown in Figure 
\ref{fig:appendix-effect-truncation-temp-sampling-disease} in the Appendix.
}
\label{fig:effect-truncation-temp-sampling}
\vspace{-10pt}  
\end{figure}

\paragraph{Effect of temperatures and sampling methods.}
LLMs generate responses through next-token prediction, meaning that sampling strategy directly influences their expressed capabilities. We consider two key factors: the temperature $\mathrm{Temp} \in \{0.7, 1\}$ and the use of nucleus sampling (denoted as $\mathrm{Nuc} \checkmark$ if applied and $\mathrm{Nuc} \times$ if not), which truncates the token distribution to the smallest set of tokens whose cumulative probability exceeds $0.9$ \citep{Holtzman2020The}.
The right two plots of Figure~\ref{fig:effect-truncation-temp-sampling} present results under these four configurations. We observe that using a higher temperature and disabling truncation yields the highest values for both $N_{\mathrm{seen}}$ and $\widehat{N}_{\mathrm{tot}}$. Therefore, we adopt a temperature of 1 and disable nucleus sampling in all experiments to best capture each model's capabilities.

\section{Discussion}

We have presented \textsc{KnowSum}, a five-step modular pipeline for estimating the unseen knowledge encoded in LLMs. By focusing on what models could output rather than what they do output, our method provides a complementary perspective to standard evaluation. We show its effectiveness across three applications: knowledge estimation, information retrieval, and diversity measurement.

Our framework opens several avenues for extension. While this work focuses on estimating the number of unseen knowledge items, similar estimators can be adapted to quantify knowledge that appears at least $s$ times \citep{hao2020optimal}. Currently, our framework focuses on well-defined, countable knowledge, such as named theorems or diseases, whose correctness can be easily verified. However, real-world knowledge is often more complex and less structured. For example, knowing the name of a theorem is not the same as understanding it, where ``understanding'' itself is difficult to define. As another example, if we shift our focus to proof techniques, it is generally hard for machines to automatically identify which techniques are used—especially since complex theorem proofs often involve multiple techniques and intricate relationships among them.
In principle, if appropriate verification and clustering methods can be developed, our pipeline could still apply in these settings. However, doing so is generally difficult, and we leave it as an open problem. Finally, a more ambitious direction is to move beyond estimation and toward extraction—actively surfacing knowledge that the model encodes but rarely generates. 
Transitioning from passive statistical inference to active knowledge discovery introduces both theoretical and practical challenges, and we view this as an important and impactful direction for future research.

\section*{Acknowledgments}
This work was supported in part by NIH grants, RF1AG063481 and U01CA274576, NSF DMS-2310679, a Meta Faculty Research Award, and Wharton AI for Business. The content is solely the responsibility of the authors and does not necessarily represent the official views of the NIH.


\bibliographystyle{plainnat}
\bibliography{bib/chatgpt,bib/LLMecology, bib/info_retrieval}

\newpage
\appendix
\newpage
\section{Proofs of Theorem \ref{thm:relation}}
\label{proof:thm-relation}

\begin{proof}[Proof of Theorem \ref{thm:relation}]
Recall that \( n \) is the total number of valid items. So we have that $n = \sum_{s \ge 1} s n_s$.
By the expression \eqref{eq:SGT} of the SGT estimator, we have 
$ \widehat{N}_{\mathrm{unseen}}(t) = \sum_{s=1}^kh_s n_s$. 
To prove the result, it suffices to show that
\begin{equation}
\label{eq:ratio-bound}
\left|\frac{\widehat{N}_{\mathrm{unseen}}(t)}{N_{\mathrm{obs}}}\right| \le  e^{\frac{kt}{t+1}}.
\end{equation}
We now proceed to prove \eqref{eq:ratio-bound}. It follows that
\begin{align*}
\frac{\widehat{N}_{\mathrm{unseen}}(t)}{ N_{\mathrm{obs}} } 
&= \frac{\sum_{s=1}^kh_s n_s}{\sum_{s \ge 1} n_s}
\le \max_{ 1 \le s \le k} \left|h_s\right|\\
&\le \max_{ 1 \le s \le k} t^s \cdot \PB\left(\mathrm{Bin}\left(k, \frac{1}{t+1}\right) \geq s\right) \\
&\overset{(a)}{\le} \max_{ 1 \le s \le k}  \left( \frac{tek }{(t+1) s} \right)^s e^{-\frac{k}{t+1}} \\
&\le  \max_{ 1 \le s \le k} \left( \frac{ek }{s} \right)^s e^{-\frac{k}{t+1}} \\
&\overset{(b)}{\le} e^{\frac{kt}{t+1}},
\end{align*}
where the inequality $(a)$ uses Lemma \ref{lem:tail} and $(b)$ uses the fact that $\left( \frac{ek }{s} \right)^s$ achieves its maximum when $s = k$ so that $\left( \frac{ek }{s} \right)^s \le e^k$.

\end{proof}

\begin{lem}[Tail bound]
\label{lem:tail}
\[
\PB\left(
\mathrm{Bin}\left(k, \frac{1}{t+1}\right)  \geq s\right) 
\leq  \left( \frac{e k}{(t+1) s} \right)^s.
\]
\end{lem}
\begin{proof}[Proof of Lemma \ref{lem:tail}]
Let \( X \sim \mathrm{Bin}(k, \frac{1}{t+1} ) \), and define the mean \( \mu = \mathbb{E}[X] = \frac{k}{t+1} \). We aim to upper bound the tail probability $\mathbb{P}(X \geq s)$ for \( s > \mu \).
To that end, we apply the multiplicative Chernoff bound:
\[
\mathbb{P}(X \geq s) \leq \left( \frac{e^{\delta}}{(1+\delta)^{1+\delta}} \right)^\mu,
\]
where \( \delta = \frac{s - \mu}{\mu} = \frac{s}{\mu} - 1 \).
Substituting this into the bound, we obtain:
\[
\mathbb{P}(X \geq s) \leq \left( \frac{e^{\frac{s}{\mu} - 1}}{\left( \frac{s}{\mu} \right)^{\frac{s}{\mu}}} \right)^\mu.
\]
Now we simplify the right-hand side. Taking logarithms:
\[
\log \mathbb{P}(X \geq s) \leq \mu \left( \frac{s}{\mu} - 1 - \frac{s}{\mu} \log \left( \frac{s}{\mu} \right) \right)
= -s \log \left( \frac{s}{\mu} \right) + s - \mu.
\]
Exponentiating both sides gives:
\[
\mathbb{P}(X \geq s) \leq \exp\left( -s \log \left( \frac{s}{\mu} \right) + s - \mu \right)
= \left( \frac{\mu}{s} \right)^s \cdot e^{s - \mu}.
\]
Finally, plugging in \( \mu = \frac{k}{t+1} \), we get:
\[
\mathbb{P}(X \geq s) \leq \left( \frac{e \cdot \mu}{s} \right)^s \cdot e^{-\mu} = \left( \frac{e k}{(t+1) s} \right)^s e^{-\frac{k}{t+1}}.
\]
\end{proof}

\section{Experiments Details}
\label{appen:experiment}

\subsection{General Setup}
In our work, we evaluate nine widely used LLMs: \texttt{ChatGPT-4o-chat} \citep{hurst2024gpt}, \texttt{ChatGPT-3.5-turbo-chat} \citep{openai2023chatgpt35}, \texttt{LLaMA-V3-70B-instruct} and \texttt{LLaMA-V3-3B-instruct} \citep{grattafiori2024llama}, \texttt{Mistral-7B-instruct-V0.1} \citep{jiang2023mistral7b}, \texttt{Qwen2.5-7B-instruct} \citep{yang2024qwen2}, \texttt{Claude-3.7-Sonnet} \citep{anthropic2025claude37}, \texttt{DeepSeek-V3} \citep{liu2024deepseek}, and \texttt{Gemini-1.5-flash} \citep{team2024gemini}.

To ensure fair comparisons, we query each model using carefully designed prompt templates (see Appendix~\ref{appen:prompts}), set the temperature to 1, and collect their responses. We then apply our estimation pipeline (Algorithm~\ref{alg:detection}) to estimate the amount of unseen knowledge.
Details on the verification and clustering procedures are provided in Appendix~\ref{appen:verification-clustering}, along with additional experimental results in the subsequent subsections.

\subsection{Used Prompts}
\label{appen:prompts}
The prompts used in Section~\ref{sec:experiment} are listed in Table~\ref{tab:prompts}, where the underlined words indicate the parts modified for sensitivity analysis.
For Application 2 Subtask 2, the prompt templates are shown in Table~\ref{tab:bioasq_qa_context_prompts}. In general, we use different templates for different types of questions to provide appropriate examples that facilitate effective in-context learning.

To test prompt sensitivity, we remove instructional content from the original prompt. In the theorem counting task, the simplified version becomes: ``Test your knowledge of mathematical theorems by listing 20 theorem names, separated by commas.'' In the human disease counting task, the simplified version becomes: ``You are a biomedical expert helping to compile a list of human diseases. Please provide 50 human diseases along with their corresponding DOID identifiers.''

\newpage
\begin{table}[p]
\centering
\caption{Prompts used across different applications and tasks.}
\label{tab:prompts}
\resizebox{0.8\linewidth}{!}{%
\begin{tabular}{@{}c|c|p{10cm}@{}}
\toprule
\textbf{Application} & \textbf{Task} & \textbf{Prompt} \\ 
\midrule

\multirow{14}{*}{\centering 
\rotatebox{90}{Application 1}
} 
& 
\multirow{7}{*}{Theorems counting} 
& 
\begin{minipage}[t]{\linewidth}\ttfamily
Test your knowledge of mathematical theorems by listing \underline{20} theorem names, separated by commas (e.g., Theorem 1, Theorem 2, Theorem 3, etc.). If a theorem is known by multiple names, choose the most classic one. Aim for variety and try to include rare or less commonly known theorems. Keep your response concise and to the point.
\end{minipage} \\
\cmidrule{2-3}
& 
\multirow{7}{*}{\centering Diseases counting} 
& 
\begin{minipage}[t]{\linewidth}\ttfamily
You are a biomedical expert helping to compile a list of human diseases.

Please provide \underline{50} human diseases along with their corresponding DOID identifiers.

Format your response as:
- Disease Name (DOID:xxxxx)

Make sure the Disease Name matches with DOID, and the DOIDs are valid from Disease Ontology.
\end{minipage} \\
\midrule

\multirow{19}{*}{\centering \rotatebox{90}{Application 2}} 
& 
\multirow{19}{*}{\centering BioASQ retrieval \cite{zhou2024using}}  
& 
\begin{minipage}[t]{\linewidth}\ttfamily
Given a question, expand into a search query for PubMed by incorporating synonyms and additional terms that would yield relevant search results from PubMed to the provided question while not being too restrictive. Assume that phrases are not stemmed; therefore, generate useful variations. Return only the query that can directly be used without any explanation text.

Question: What is the mode of action of Molnupiravir?\\
Query: Molnupiravir AND ("mode of action" OR mechanism).

Question: Is dapagliflozin effective for COVID-19?\\
Query: dapagliflozin AND (COVID-19 OR SARS-CoV-2 OR coronavirus) AND (efficacy OR effective OR treatment).

Question: Name monoclonal antibody against SLAMF7.\\
Query: "SLAMF7" AND ("monoclonal antibody" OR "monoclonal antibodies").

Question: \{context\}\\
Query: \{body\}
\end{minipage} \\
\midrule

\multirow{7}{*}{\centering\rotatebox{90}{Application 3}} 
& 
\multirow{3}{*}{\centering LLMs applications} 
& 
\begin{minipage}[t]{\linewidth}\ttfamily
Briefly describe one real or imagined application of large language models (LLMs). Keep the response short and simple, and explain why it's useful.
\end{minipage} \\
\cmidrule{2-3}
& \multirow{4}{*}{\centering Dream jobs}  & 
\begin{minipage}[t]{\linewidth}\ttfamily
Imagine that a large language model (LLM) could dream like a human. Describe, in one or two sentences, what kind of job it might wish to have. Be creative, and explain briefly why this job would appeal to an LLM.
\end{minipage} \\
\bottomrule
\end{tabular}}
\end{table}

\clearpage
\begin{table}[p]
\centering
\caption{Context prompt templates used for each question type in Application 2, Subtask 2 (Question Answering) of the BioASQ challenge.}
\label{tab:bioasq_qa_context_prompts}
\resizebox{0.6\linewidth}{!}{%
\begin{tabular}{@{}c|p{12cm}@{}}
\toprule
\textbf{Question Type} & \textbf{Prompt Templates} \\
\midrule

\multirow{20}{*}{\centering Yes/No}
 & 
\begin{minipage}[t]{\linewidth}\ttfamily
Here are some example yesno questions with answers.\\
\#\#\#Context: Papilins are homologous, secreted extracellular matrix proteins which share a common order of protein domains.\\
Question: Is the protein Papilin secreted?\\
Ideal answer: Yes, papilin is a secreted protein\\
Exact answer: yes\\
\#\#\#Context: Most lncRNAs are under lower sequence constraints than protein-coding genes and lack conserved secondary structures, making it hard to predict them computationally.\\
Question: Are long non coding RNAs as conserved in sequence as protein coding genes?\\
Ideal answer: No. Most long non coding RNAs are under lower sequence constraints than protein-coding genes.\\
Exact answer: no\\
\#\#\#Now answer the following yesno question:\\
\#\#\#Context: \{context\}\\
Question: \{body\}\\
Your Answer should be the following format:\\
Ideal answer:<Your Ideal Answer>. Exact answer:<Your Exact Answer>.
\end{minipage} \\
\midrule

\multirow{15}{*}{\centering Summary} & 
\begin{minipage}[t]{\linewidth}\ttfamily
Here are some example summary questions with answers.\\
\#\#\#Context: Hirschsprung disease (HSCR) is a multifactorial, non-mendelian disorder in which rare high-penetrance coding sequence mutations in the receptor tyrosine kinase RET contribute to risk in combination with mutations at other genes.\\
Question: Is Hirschsprung disease a mendelian or a multifactorial disorder?\\
Answer: Coding sequence mutations in RET, GDNF, EDNRB, EDN3, and SOX10 are involved in the development of Hirschsprung disease. The majority of these genes was shown to be related to Mendelian syndromic forms...\\
\#\#\#Now answer the following summary question:\\
Context: \{context\}\\
Question: \{body\}\\
Answer:
\end{minipage} \\
\midrule

\multirow{13}{*}{\centering List}  & 
\begin{minipage}[t]{\linewidth}\ttfamily
Here are some example list questions with answers.\\
\#\#\#Context: The FGFR3 P250R mutation was the single largest contributor (24\%) to the genetic group.\\
Question: Which human genes are more commonly related to craniosynostosis?\\
Ideal answer: The genes that are most commonly linked to craniosynostoses are...\\
Exact answer: FGFR3;FGFR2;FGFR1;...\\
\#\#\#Now answer the following list question:\\
Context: \{context\}\\
Question: \{body\}\\
Ideal answer:<Your Ideal Answer>. Exact answer:<Your Exact Answer>.
\end{minipage} \\
\midrule

\multirow{13}{*}{\centering Factoid} & 
\begin{minipage}[t]{\linewidth}\ttfamily
Here are some example factoid questions with answers.\\
\#\#\#Context: Ewing sarcoma is the second most common bone malignancy in children and young adults.\\
Question: Which fusion protein is involved in the development of Ewing sarcoma?\\
Ideal answer: Ewing sarcoma is driven by oncogenic fusion protein EWS/FLI1...\\
Exact answer: EWS;FLI1\\
\#\#\#Now answer the following factoid question:\\
Context: \{context\}\\
Question: \{body\}\\
Ideal answer:<Your Ideal Answer>. Exact answer:<Your Exact Answer>.
\end{minipage} \\

\bottomrule
\end{tabular}}
\end{table}

\clearpage
\subsection{Details of Verification and Clustering}
\label{appen:verification-clustering}

 \subsubsection{Application 1: Knowledge Estimation}

 \paragraph{Mathematical theorems.}
In the \textbf{verification} step, we sequentially cross-reference each generated item against three external sources: Wikipedia, MathSciNet, and ProofWiki. The process stops as soon as a match is found. Two theorems are considered identical if they are verified by the same webpage or share the same URL.

In the \textbf{clustering} step, since a single theorem may have multiple equivalent names, we assign a canonical identifier to each theorem by extracting a standardized name from the corresponding URL. For Wikipedia and ProofWiki, we parse the page title from the URL path, convert it to lowercase, and replace underscores with spaces. For MathSciNet, we extract the search query parameter. All extracted names are further passed through a normalization routine to reduce inconsistencies due to formatting or spelling variations. This process ensures that each valid theorem instance is mapped to a unique and consistent name across sources.

\paragraph{Human Disease Oncology.}

In the \textbf{verification} step, each LLM-generated disease name is evaluated by matching it against official disease terms and their synonyms. These synonyms are drawn from the ``term.synonym'' field in the OBO file provided by the Disease Ontology. Matching is performed using fuzzy string comparison via the ``rapidfuzz'' Python package, with a similarity threshold of 0.9. The verification process terminates once a valid match is found.

In the \textbf{clustering} step, all matched terms are grouped according to their corresponding Disease Ontology Identifier (DOID). For example, if an LLM generates a disease name along with two of its synonyms, all three generated names will be mapped to the same DOID, and the count for that DOID will be incremented by three.

 \subsubsection{Application 2: Information Retrieval}

We evaluate the information retrieval capabilities of LLMs using the official \texttt{Task12BGolden} test set from the BioASQ Challenge \citep{krithara2023bioasq}. Each question in this dataset is annotated with a set of relevant biomedical documents, each uniquely identified by a PubMed ID. For each PubMed ID (or equivalently each biomedical document), we retrieve the associated Medical Subject Headings (MeSH) terms using the NCBI E-utilities API (\url{https://eutils.ncbi.nlm.nih.gov/entrez/eutils/efetch.fcgi}).  Throughout this application, normalized MeSH terms serve as the countable unit of biomedical knowledge—analogous to the role of theorems in Application 1. In the \textbf{clustering} step, keywords were normalized to a unique form by splitting the original name on the slash delimiter (``/'') and retaining the last segment in lowercase.

\paragraph{Subtask 1: Document Retrieval.}
This subtask assesses an LLM's ability to formulate effective Boolean queries for retrieving relevant scientific articles from the PubMed database. To measure unseen knowledge, we estimate the number of additional, correct MeSH terms the model could uncover if queried with semantically similar inputs. During the \textbf{generation} phase, we prompt the LLM with a BioASQ question using a fixed template (see Table~\ref{tab:prompts}), and use the resulting query to search PubMed via the E-utilities API (\texttt{db = ``pubmed''}). In the \textbf{verification} phase, we check whether retrieved PubMed IDs match any gold-standard references. If a match occurs, we accumulate all MeSH terms associated with the matching documents. Because each article can have multiple MeSH terms, and each question can be linked to multiple documents, this setup provides a rich space of retrievable knowledge.

\paragraph{Subtask 2: Question Answering.}

The second subtask evaluates LLMs' ability to answer questions based on retrieved biomedical content. In the \textbf{generation} phase, to ensure a fair comparison, all models receive the same gold-standard document snippets as input. This setting makes use of in-context learning. See Table~\ref{tab:bioasq_qa_context_prompts} for the prompt template.
For each BioASQ question, the model generates two types of responses: (1) an \textit{exact answer}, which may be of type yes/no, list, or factoid, and is evaluated using standard metrics such as accuracy, precision, recall, and F1 scores; and (2) an \textit{ideal answer}, which is a free-form summary evaluated using ROUGE.

 In the \textbf{verification} phase, for yes/no and factoid questions, we compare the LLM-generated answer directly against the reference label. For list and summary questions, we apply threshold-based criteria to assess correctness. Unlike document-level evaluation in Subtask 1, this subtask operates at the question level. If the model's output is deemed correct, we credit all the MeSH terms associated with that question's gold documents as successfully recovered. We then apply our unseen knowledge estimation framework to infer the number of MeSH terms the model could generate if exposed to additional questions.

 \subsubsection{Application 3: Diversity Measure}

We don't verify responses in this application. We describe how we cluster responses: each response is embedded using OpenAI's \texttt{text-embedding-ada-002} model, which maps text to a 1536-dimensional semantic vector. Two responses are considered equivalent if the distance between their embeddings is below a chosen threshold. To set this threshold, we compute the 10-nearest neighbor distances for each response within each LLM, aggregate the distances across all models, and use the $q$-quantile of the combined set. We evaluate $q \in \{0.2, 0.3, 0.5, 0.7\}$.

\subsection{Verification Criteria}

In the verification step for Application 1, we apply two levels of verification criteria: one strict and one more relaxed. Since all valid knowledge items—whether theorems or human diseases—are sourced from external databases, we visualize the scope of these datasets to assess how permissive the relaxed verification criterion can be while still maintaining validity in Figure \ref{fig:subcategory_distribution}.

\begin{figure}[th]
\vspace{-1em}
\centering
\includegraphics[width=\columnwidth]{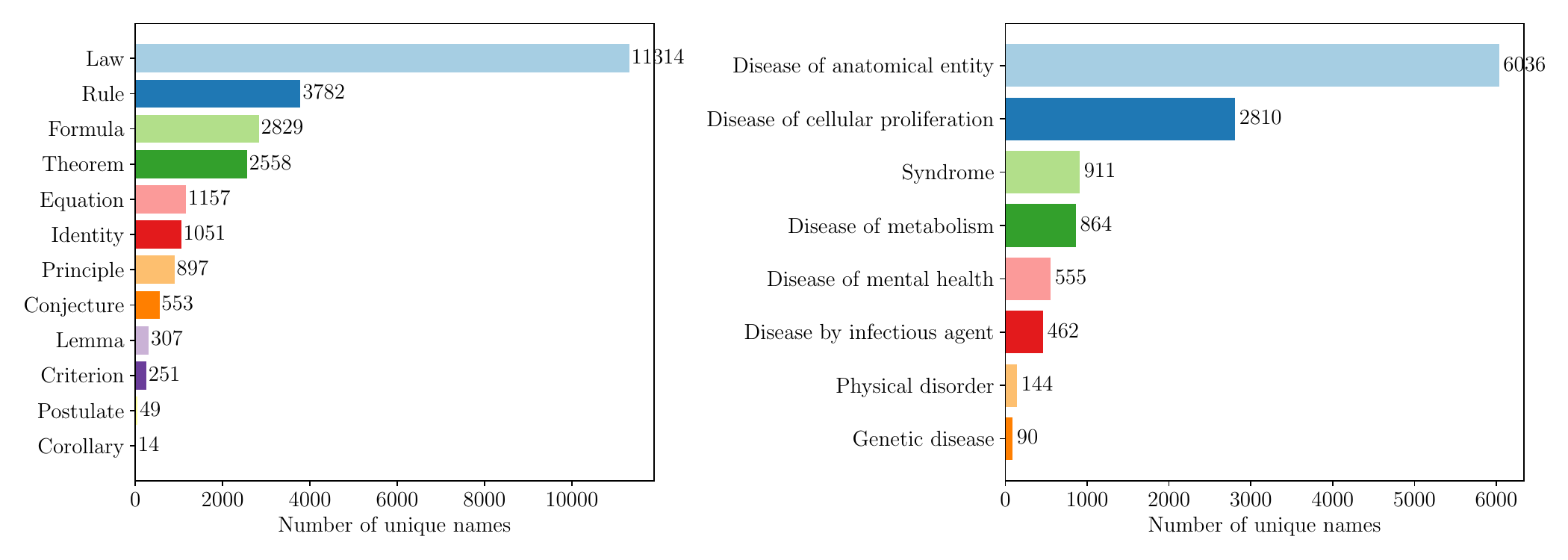}
\caption{Distribution of math concepts (\textbf{left}) and human disease subcategories (\textbf{right}). These distributions are used as filtering criteria to exclude invalid knowledge items under the strict criteria setting in Application 1.}
\label{fig:subcategory_distribution}
\vspace{-1em}
\end{figure}

\paragraph{Mathematical theorems.} 
For theorems, the strict criterion requires that the final name (extracted from the URL) contains the word ``theorem'', while the relaxed criterion accepts any name that includes one of 12 math concepts (i.e., ``theorem,'' ``lemma,'' ``law,'' ``principle,'' ``formula,'' ``criterion,'' ``identity,'' ``conjecture,'' ``rule,'' ``equation,'' ``postulate,'' or ``corollary.''). The left plot in Figure~\ref{fig:subcategory_distribution} shows the distribution of different theorem types in the reference databases. Notably, law is the most common type, with 11,314 entries, while the theorem type includes only 2,558 entries.

\paragraph{Human diseases.} 
For human diseases, the strict criterion includes only the ``anatomical disease'' category—i.e., diseases that affect specific body parts. The relaxed criterion allows any disease listed in the Disease Ontology database, regardless of category. The right plot in Figure~\ref{fig:subcategory_distribution} displays the distribution of disease types in the reference databases. Notably, anatomical disease is the most common category, with 6036 entries. 

\subsection{Used Truncation Levels}
\label{appen:truncation-levels}

 All the truncation levels used in different applications are collected in Table \ref{tab:truncartion-levels}.

\begin{table}[th]
\caption{Truncation levels used in different applications.}
\label{tab:truncartion-levels}
\centering
\resizebox{\linewidth}{!}{%
\begin{tabular}{l|ccccc}
\toprule
\textbf{Model} 
& \makecell{\textbf{Application 1} \\ Theorems}
& \makecell{\textbf{Application 1} \\ Human diseases}
& \makecell{\textbf{Application 2} \\ Retrieval}
& \makecell{\textbf{Application 2} \\ Question Answer}
& \makecell{\textbf{Application 3} \\ Diversity} \\
\midrule
\ding{172} \texttt{ChatGPT-4o-chat} & 8 & 8 & 8 & 8 & 6 \\
\ding{173} \texttt{ChatGPT-3.5-turbo-chat} & 8 & 8 & 8 & 8 & 6 \\
\ding{174} \texttt{LLaMA-V3-70B-instruct} & 8 & 8 & 8 & 8 & 6 \\
\ding{175} \texttt{LLaMA-V3-3B-instruct} & 6 & 8 & 8 & 8 & 6 \\
\ding{176} \texttt{Mistral-7B-instruct-v0.1} & 6 & 10 & 10 & 10 & 6 \\
\ding{177} \texttt{Qwen2.5-7B-instruct} & 8 & 8 & 6 & 6 & 6 \\
\ding{178} \texttt{Claude-3.7-Sonnet} & 8 & 8 & 8 & 8 & 6 \\
\ding{179} \texttt{DeepSeek-V3} & 8 & 8 & 8 & 8& 6 \\
\ding{180} \texttt{Gemini-1.5-flash} & 8 & 6 & 6 & 6 & 6 \\
\bottomrule
\end{tabular}}
\vspace{-1em}
\end{table}

\subsection{Cross-Validation of Unseen Human Disease Estimates}

We present the human disease counting results from the held-out validation analysis, complementing the theorem results shown in Figure~\ref{fig:sec-5-1-math_theorem_validation}. As shown in Figure~\ref{fig:appendix-disease_validation}, the estimator achieves similarly strong performance. The estimated counts of unseen diseases closely match the ground truth, confirming that the SGT estimator yields highly accurate predictions across a wide range of extrapolation factors.

\begin{figure}[th]
\vspace{-10pt}  
\centering 
\includegraphics[width=\textwidth]{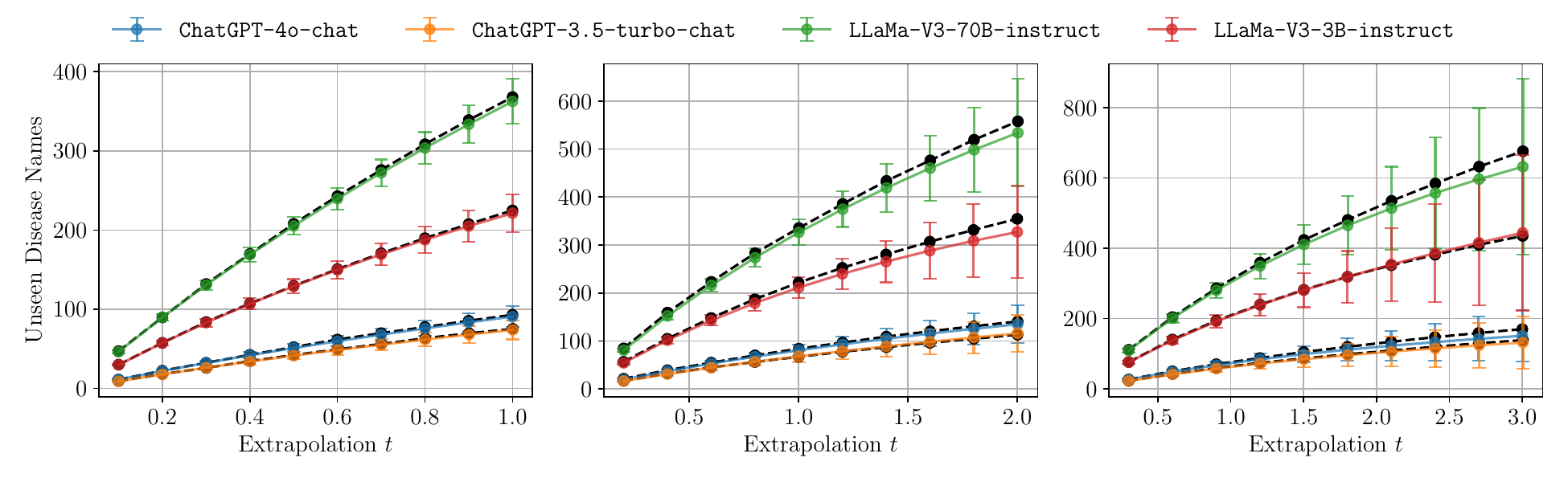}
\vspace{-15pt}  
\caption{SGT estimates (colored curves) versus ground truth (black dotted lines) for disease estimation. From left to right, the observed fraction $r_{\mathrm{obs}}$ is $1/2$, $1/3$, and $1/4$.}
\label{fig:appendix-disease_validation}
\vspace{-5pt}  
\end{figure}

\subsection{Additional Results on Extrapolation and Prompts}

We present the counterpart of human diseases of Figure \ref{fig:effect-t-prompt}. We examine two key factors: the clarity of the instruction ($\mathrm{Inst}\uparrow$ or $\mathrm{Inst}\downarrow$) and the number of requested responses per query ($N_{\mathrm{ans}} \in \{30, 50\}$) As shown in Figure \ref{fig:appendix-disease_effect-t-prompt} left two plots, we can see that the estimator converges at $t=100$ for \texttt{ChatGPT-4o-Chat}, $t=80$ for $\texttt{LLaMa-V3-70b-instruct}$ and $\texttt{LLaMa-V3-3b-instruct}$, and $t=25$ for \texttt{ChatGPT-3.5-turbo-Chat}. From the right two plots of Figure \ref{fig:appendix-disease_effect-t-prompt}, we can see that clear instruction and high number of response can increase the diversity and elicit more LLM Unseen knowledge.

\begin{figure}[th]
\vspace{-8pt}  
\centering 
\includegraphics[width=\textwidth]{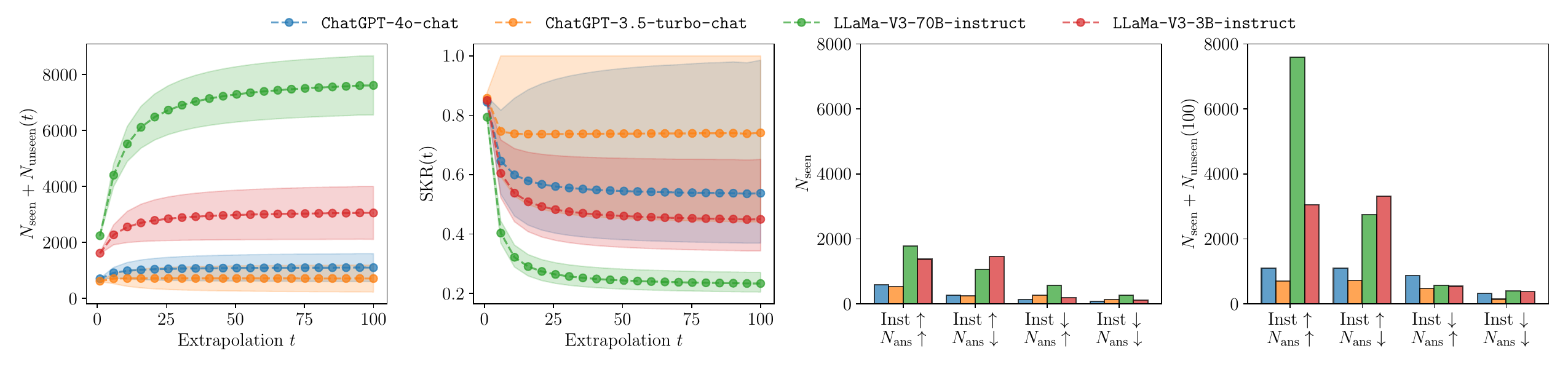}
\vspace{-15pt}  
\caption{Effects of the extrapolation factor $t$ (left two plots) and prompts (right two plots) for disease estimation.
}
\label{fig:appendix-disease_effect-t-prompt}
\vspace{-5pt}  
\end{figure}

\subsection{Additional Results on Truncation Levels and Sampling Strategies}

We present additional results for human diseases in Figure~\ref{fig:appendix-effect-truncation-temp-sampling-disease}, complementing Figure~\ref{fig:effect-truncation-temp-sampling}. The left two plots show that the estimator is relatively robust to different values of the truncation level $k$. The right two plots indicate that using a temperature of 1 and disabling nucleus sampling consistently yields higher values of both $N_{\mathrm{seen}}$ and $N_{\mathrm{tot}}$ across models.
 
\begin{figure}[H]
\vspace{-8pt}  
\centering 
\includegraphics[width=\textwidth]{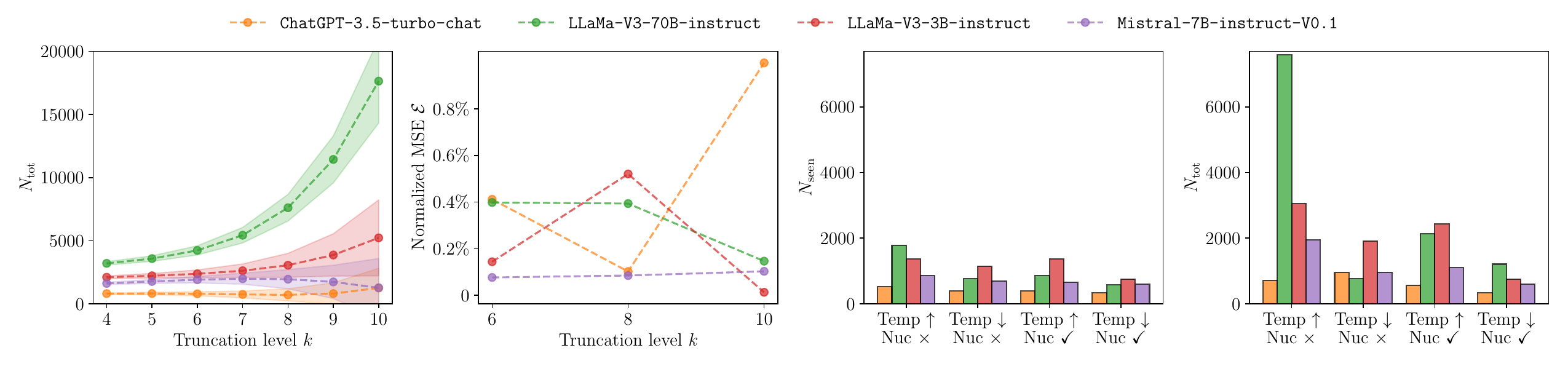}
\vspace{-15pt}  
\caption{Effects of the truncation level $k$ (left two plots) and sampling strategies (right two plots) for disease estimation. 
}
\label{fig:appendix-effect-truncation-temp-sampling-disease}
\vspace{-8pt}  
\end{figure}

\subsection{Responses Analysis}
\label{appen:responses}

Table~\ref{tab:model_topk_theorems_counts} presents the most frequently generated mathematical theorems across nine LLMs, each prompted to list 20 theorem names over 30,000 runs. The table reports the top-8 theorems per model, along with their absolute counts and relative frequencies. Classical results such as the ``Pythagorean Theorem,'' ``Fermat's Last Theorem,'' and ``Gödel's Incompleteness Theorems'' appear consistently across nearly all models, suggesting that these iconic theorems are embedded as core knowledge. At the same time, notable differences emerge: \texttt{Qwen2.5-7B-instruct} strongly favors the ``Fundamental Theorem of Calculus,'' while \texttt{LLaMA-V3-3B-instruct} uniquely highlights the ``Sylvester–Gallai Theorem'' and ``Navier–Stokes Equations.'' These variations likely reflect differences in pretraining data coverage or model specialization.

Table~\ref{tab:model_topk_human_disease_counts} shows the top-8 most frequently generated human diseases from each LLM, prompted to list 50 disease names over 3000 runs. As in the theorem case, many models converge on a core set of common diseases, including ``Alzheimer's Disease,'' ``Parkinson's Disease,'' ``Asthma,'' and ``Multiple Sclerosis.'' These results reflect the centrality of these conditions in biomedical corpora. Nonetheless, distinct model behaviors are evident: for example, \texttt{Gemini-1.5-flash} emphasizes ``Coronary Artery Disease'' and ``Type 2 Diabetes Mellitus,'' while \texttt{LLaMA-V3-3B-instruct} features rarer conditions like ``Maturity-Onset Diabetes of the Young Type 1.'' Such differences likely stem from variations in training data granularity and domain emphasis.

Table~\ref{tab:model_topk_anatomical_disease_counts} focuses on anatomical diseases and follows the same experimental setup. Common conditions like ``Hepatitis,'' ``Alzheimer's Disease,'' ``Parkinson's Disease,'' and ``Asthma'' again dominate the responses, with ``Hepatitis'' especially prominent across models such as \texttt{ChatGPT-4o-chat}, \texttt{DeepSeek-V3}, and \texttt{Claude-3.7-Sonnet}. However, each model still exhibits unique tendencies: for instance, \texttt{LLaMA-V3-70B-instruct} frequently generates ``Frontotemporal Dementia'' and ``Chronic Kidney Disease,'' while \texttt{Mistral-7B-Instruct-v0.1} leans toward ``Cerebrovascular Disease'' and ``Heart Disease.'' These findings further underscore how LLMs encode not only general but also model-specific biomedical knowledge, shaped by their underlying training regimes.

\begin{table}[th]
\caption{Top-8 theorem outputs per model with their counts and fractions.}
\label{tab:model_topk_theorems_counts}
\centering
\resizebox{\linewidth}{!}{%
\begin{tabular}{@{}l|ll@{}}
\toprule
\textbf{Model} & \multicolumn{2}{c}{\textbf{Top-8 Outputs with Counts and Fractions}} \\
\midrule
\multirow{4}{*}{\ding{172} \texttt{ChatGPT-4o-chat}} 
  & Pythagorean Theorem (29848, 0.0557) & Brouwer Fixed Point Theorem (29813, 0.0557) \\
  & Gödel Incompleteness Theorems (29807, 0.0557) & Noether Theorem (29583, 0.0552) \\
  & Fermat Last Theorem (29337, 0.0548) & Ramsey Theorem (25767, 0.0481) \\
  & Green Theorem (19582, 0.0366) & Cantor Theorem (18837, 0.0352) \\
\midrule
\multirow{4}{*}{\ding{173} \texttt{ChatGPT-3.5-turbo-chat}} 
  & Pythagorean Theorem (32631, 0.0599) & Fermat Last Theorem (29370, 0.0539) \\
  & Bayes Theorem (25364, 0.0465) & Fundamental Theorem of Calculus (23909, 0.0439) \\
  & Gödel Incompleteness Theorems (20820, 0.0382) & Brouwer Fixed Point Theorem (20291, 0.0372) \\
  & Bolzano Weierstrass Theorem (18585, 0.0341) & Jordan Curve Theorem (17976, 0.0330) \\
\midrule
\multirow{4}{*}{\ding{174} \texttt{LLaMA-V3-70B-instruct}} 
  & Fundamental Theorem of Algebra (19825, 0.0406) & Fermat Last Theorem (18687, 0.0382) \\
  & Pythagorean Theorem (18214, 0.0373) & Brouwer Fixed Point Theorem (13398, 0.0274) \\
  & Gödel Incompleteness Theorems (13089, 0.0268) & Gauss Bonnet Theorem (12233, 0.0250) \\
  & Jordan Curve Theorem (11030, 0.0226) & Intermediate Value Theorem (9936, 0.0203) \\
\midrule
\multirow{4}{*}{\ding{175} \texttt{LLaMA-V3-3B-instruct}} 
  & Fermat Last Theorem (30185, 0.0704) & Sylvester Gallai Theorem (16568, 0.0387) \\
  & Kepler Conjecture (14174, 0.0331) & Poincaré Conjecture (13602, 0.0317) \\
  & Brouwer Fixed Point Theorem (11178, 0.0261) & Euler Identity (10231, 0.0239) \\
  & Lagrange Theorem (9576, 0.0223) & Navier Stokes Equations (9028, 0.0211) \\
\midrule
\multirow{4}{*}{\ding{176} \texttt{Mistral-7B-Instruct-v0.1}} 
  & Pythagorean Theorem (868, 0.0501) & Fermat Last Theorem (665, 0.0384) \\
  & List of Theorems (574, 0.0332) & Uncertainty Principle (486, 0.0281) \\
  & Gödel Incompleteness Theorems (459, 0.0265) & Theorem (279, 0.0161) \\
  & Gauss Bonnet Theorem (279, 0.0161) & Sylow Theorems (236, 0.0136) \\
\midrule
\multirow{4}{*}{\ding{177} \texttt{Qwen2.5-7B-instruct}} 
  & Fundamental Theorem of Calculus (14258, 0.1500) & Fundamental Theorem of Arithmetic (5276, 0.0555) \\
  & Gödel Incompleteness Theorems (3159, 0.0332) & Modularity Theorem (2227, 0.0234) \\
  & Dirichlet Approximation Theorem (2142, 0.0225) & Mean Value Theorem (1754, 0.0185) \\
  & Poincaré Bendixson Theorem (1668, 0.0175) & Law of Cosines (1641, 0.0173) \\
\midrule
\multirow{4}{*}{\ding{178} \texttt{Claude-3.7-Sonnet}} 
  & Brouwer Fixed Point Theorem (5008, 0.0545) & Pythagorean Theorem (4994, 0.0543) \\
  & Fermat Last Theorem (4994, 0.0543) & Gödel Incompleteness Theorems (4994, 0.0543) \\
  & Bayes Theorem (4860, 0.0528) & Gauss Bonnet Theorem (4779, 0.0520) \\
  & Borsuk Ulam Theorem (4643, 0.0505) & Sylow Theorems (4056, 0.0441) \\
\midrule
\multirow{4}{*}{\ding{179} \texttt{DeepSeek-V3}} 
  & Pythagorean Theorem (29999, 0.0526) & Brouwer Fixed Point Theorem (29982, 0.0526) \\
  & Hahn Banach Theorem (29835, 0.0523) & Gödel Incompleteness Theorems (29402, 0.0516) \\
  & Central Limit Theorem (29223, 0.0513) & Fundamental Theorem of Algebra (27156, 0.0476) \\
  & Noether Theorem (27038, 0.0474) & Heine Borel Theorem (26967, 0.0473) \\
\midrule
\multirow{4}{*}{\ding{180} \texttt{Gemini-1.5-flash}} 
  & Pythagorean Theorem (30000, 0.0651) & Prime Number Theorem (29989, 0.0651) \\
  & Central Limit Theorem (29936, 0.0649) & Jordan Curve Theorem (29825, 0.0647) \\
  & Mean Value Theorem (29225, 0.0634) & Law of Cosines (28753, 0.0624) \\
  & Bolzano Weierstrass Theorem (28101, 0.0610) & Intermediate Value Theorem (27045, 0.0587) \\
\bottomrule
\end{tabular}%
 }
\end{table}

\newpage
\begin{table}[th]
\caption{Top-8 Human diseases outputs per model with their counts and fractions.}
\label{tab:model_topk_human_disease_counts}
\centering
\resizebox{\linewidth}{!}{%
\begin{tabular}{@{}l|ll@{}}
\toprule
\textbf{Model} & \multicolumn{2}{c}{\textbf{Top-8 Outputs with Counts and Fractions}} \\
\midrule
\multirow{4}{*}{\ding{172} \texttt{ChatGPT-4o-chat}} 
  & Alzheimer'S Disease (3794, 0.0262) & Parkinson'S Disease (2995, 0.0207) \\
   & Asthma (2966, 0.0205) & Multiple Sclerosis (2941, 0.0203) \\
   & Rheumatoid Arthritis (2919, 0.0201) & Hypertension (2885, 0.0199) \\
   & Crohn'S Disease (2879, 0.0199) & Tuberculosis (2843, 0.0196) \\
\midrule
\multirow{4}{*}{\ding{173} \texttt{ChatGPT-3.5-turbo-chat}} 
   & Alzheimer'S Disease (2969, 0.0221) & Parkinson'S Disease (2888, 0.0215) \\
   & Asthma (2826, 0.0211) & Rheumatoid Arthritis (2793, 0.0208) \\
   & Hypertension (2786, 0.0208) & Breast Cancer (2784, 0.0208) \\
   & Leukemia (2772, 0.0207) & Multiple Sclerosis (2761, 0.0206) \\
\midrule
\multirow{4}{*}{\ding{174} \texttt{LLaMA-V3-70B-instruct}} 
  & Alzheimer'S Disease (2376, 0.0211) & Breast Cancer (2243, 0.0199) \\
   & Cystic Fibrosis (1971, 0.0175) & Asthma (1935, 0.0172) \\
   & Chronic Obstructive Pulmonary Disease (1891, 0.0168) & Parkinson'S Disease (1884, 0.0167) \\
   & Multiple Sclerosis (1861, 0.0165) & Amyotrophic Lateral Sclerosis (1853, 0.0164) \\
\midrule
\multirow{4}{*}{\ding{175} \texttt{LLaMA-V3-3B-instruct}} 
 & Maturity-Onset Diabetes Of The Young Type 1 (3013, 0.0260) & Epilepsy (2481, 0.0214) \\
   & Hypertension (2397, 0.0207) & Hypothyroidism (2278, 0.0196) \\
   & Parkinson'S Disease (2272, 0.0196) & Autism Spectrum Disorder (2187, 0.0189) \\
   & Alzheimer'S Disease (2148, 0.0185) & Multiple Sclerosis (1950, 0.0168) \\
\midrule
\multirow{4}{*}{\ding{176} \texttt{Mistral-7B-Instruct-v0.1}} 
 & Alzheimer'S Disease (1399, 0.0369) & Asthma (1304, 0.0344) \\
   & Parkinson'S Disease (1200, 0.0316) & Diabetes Mellitus (1090, 0.0287) \\
   & Cancer (909, 0.0239) & Hypertension (876, 0.0231) \\
   & Multiple Sclerosis (843, 0.0222) & Cerebrovascular Disease (808, 0.0213) \\
\midrule
\multirow{4}{*}{\ding{177} \texttt{Qwen2.5-7B-instruct}} 
   & Maturity-Onset Diabetes Of The Young Type 1 (3295, 0.0243) & Parkinson'S Disease (2941, 0.0217) \\
   & Rheumatoid Arthritis (2919, 0.0215) & Multiple Sclerosis (2901, 0.0214) \\
   & Asthma (2865, 0.0211) & Alzheimer'S Disease (2842, 0.0210) \\
   & Epilepsy (2649, 0.0195) & Hypertension (2482, 0.0183) \\
\midrule
\multirow{4}{*}{\ding{178} \texttt{Claude-3.7-Sonnet}} 
  & Parkinson'S Disease (1088, 0.0202) & Alzheimer'S Disease (1087, 0.0202) \\
   & Multiple Sclerosis (1086, 0.0202) & Rheumatoid Arthritis (1086, 0.0202) \\
   & Hypertension (1086, 0.0202) & Psoriasis (1086, 0.0202) \\
   & Schizophrenia (1086, 0.0202) & Osteoarthritis (1086, 0.0202) \\
\midrule
\multirow{4}{*}{\ding{179} \texttt{DeepSeek-V3}} 
  & Asthma (2831, 0.0203) & Alzheimer'S Disease (2814, 0.0202) \\
   & Rheumatoid Arthritis (2813, 0.0202) & Cystic Fibrosis (2813, 0.0202) \\
   & Malaria (2813, 0.0202) & Tuberculosis (2813, 0.0202) \\
   & Huntington'S Disease (2813, 0.0202) & Parkinson'S Disease (2813, 0.0202) \\
\midrule
\multirow{4}{*}{\ding{180} \texttt{Gemini-1.5-flash}} 
  & Rheumatoid Arthritis (3002, 0.0202) & Asthma (3001, 0.0202) \\
   & Alzheimer'S Disease (3000, 0.0202) & Parkinson'S Disease (3000, 0.0202) \\
   & Multiple Sclerosis (2999, 0.0202) & Schizophrenia (2994, 0.0201) \\
   & Coronary Artery Disease (2984, 0.0201) & Type 2 Diabetes Mellitus (2979, 0.0200) \\
\midrule
\end{tabular}%
 }
\end{table}

\begin{table}[th]
\caption{Top-8 Anatomical diseases outputs per model with their counts and fractions.}
\label{tab:model_topk_anatomical_disease_counts}
\centering
\resizebox{\linewidth}{!}{%
\begin{tabular}{@{}l|ll@{}}
\toprule
\textbf{Model} & \multicolumn{2}{c}{\textbf{Top-8 Outputs with Counts and Fractions}} \\
\midrule
\multirow{4}{*}{\ding{172} \texttt{ChatGPT-4o-chat}} 
  & Hepatitis (4874, 0.0691) & Alzheimer'S Disease (3794, 0.0538) \\
   & Parkinson'S Disease (2995, 0.0425) & Asthma (2966, 0.0421) \\
   & Multiple Sclerosis (2941, 0.0417) & Rheumatoid Arthritis (2919, 0.0414) \\
   & Hypertension (2885, 0.0409) & Crohn'S Disease (2879, 0.0408) \\
   \midrule
\multirow{4}{*}{\ding{173} \texttt{ChatGPT-3.5-turbo-chat}} 
  & Hepatitis (3140, 0.0443) & Alzheimer'S Disease (2969, 0.0419) \\
   & Parkinson'S Disease (2888, 0.0407) & Asthma (2826, 0.0399) \\
   & Rheumatoid Arthritis (2793, 0.0394) & Hypertension (2786, 0.0393) \\
   & Multiple Sclerosis (2761, 0.0389) & Osteoporosis (2704, 0.0381) \\
\midrule
\multirow{4}{*}{\ding{174} \texttt{LLaMA-V3-70B-instruct}} 
  & Hepatitis (2572, 0.0421) & Alzheimer'S Disease (2376, 0.0389) \\
   & Asthma (1935, 0.0317) & Frontotemporal Dementia (1925, 0.0315) \\
   & Chronic Obstructive Pulmonary Disease (1891, 0.0310) & Parkinson'S Disease (1884, 0.0308) \\
   & Multiple Sclerosis (1861, 0.0305) & Chronic Kidney Disease (1689, 0.0277) \\
\midrule
\multirow{4}{*}{\ding{175} \texttt{LLaMA-V3-3B-instruct}} 
& Epilepsy (2481, 0.0331) & Hypertension (2397, 0.0320) \\
   & Hypothyroidism (2278, 0.0304) & Parkinson'S Disease (2272, 0.0303) \\
   & Alzheimer'S Disease (2148, 0.0287) & Multiple Sclerosis (1950, 0.0260) \\
   & Osteoarthritis (1866, 0.0249) & Rheumatoid Arthritis (1855, 0.0248) \\
\midrule
\multirow{4}{*}{\ding{176} \texttt{Mistral-7B-Instruct-v0.1}} 
 & Alzheimer'S Disease (1399, 0.0608) & Asthma (1304, 0.0567) \\
   & Parkinson'S Disease (1200, 0.0522) & Hypertension (876, 0.0381) \\
   & Multiple Sclerosis (843, 0.0366) & Cerebrovascular Disease (808, 0.0351) \\
   & Heart Disease (627, 0.0273) & Epilepsy (611, 0.0266) \\
\midrule
\multirow{4}{*}{\ding{177} \texttt{Qwen2.5-7B-instruct}} 
   & Parkinson'S Disease (2941, 0.0343) & Rheumatoid Arthritis (2919, 0.0340) \\
   & Multiple Sclerosis (2901, 0.0338) & Asthma (2865, 0.0334) \\
   & Alzheimer'S Disease (2842, 0.0331) & Epilepsy (2649, 0.0309) \\
   & Hypertension (2482, 0.0289) & Hepatitis (2382, 0.0278) \\
\midrule
\multirow{4}{*}{\ding{178} \texttt{Claude-3.7-Sonnet}} 
  & Hepatitis (1671, 0.0512) & Parkinson'S Disease (1088, 0.0334) \\
   & Alzheimer'S Disease (1087, 0.0333) & Osteoarthritis (1086, 0.0333) \\
   & Epilepsy (1086, 0.0333) & Multiple Sclerosis (1086, 0.0333) \\
   & Psoriasis (1086, 0.0333) & Asthma (1086, 0.0333) \\
\midrule
\multirow{4}{*}{\ding{179} \texttt{DeepSeek-V3}} 
  & Hepatitis (5526, 0.0715) & Asthma (2831, 0.0366) \\
   & Alzheimer'S Disease (2814, 0.0364) & Parkinson'S Disease (2813, 0.0364) \\
   & Rheumatoid Arthritis (2813, 0.0364) & Huntington'S Disease (2813, 0.0364) \\
   & Multiple Sclerosis (2812, 0.0364) & Frontotemporal Dementia (2807, 0.0363) \\
\midrule
\multirow{4}{*}{\ding{180} \texttt{Gemini-1.5-flash}} 
  & Hepatitis (3361, 0.0390) & Rheumatoid Arthritis (3002, 0.0348) \\
   & Asthma (3001, 0.0348) & Alzheimer'S Disease (3000, 0.0348) \\
   & Parkinson'S Disease (3000, 0.0348) & Multiple Sclerosis (2999, 0.0348) \\
   & Coronary Artery Disease (2984, 0.0346) & Crohn'S Disease (2979, 0.0346) \\
\midrule
\end{tabular}%
 }
\end{table}

\subsection{Traditional Evaluation of Biomedical Information Retrieval}
\label{appen:traditional-IR}

We report the results of traditional evaluation metrics for Application 2 in this subsection. In our setup, each biomedical question is paired with a set of ground-truth documents, each annotated with expert-curated MeSH keywords. These annotations serve as a concrete proxy for the biomedical knowledge a model retrieves during document search or question answering. Standard IR metrics—such as precision, recall, F1 score, and mean reciprocal rank (MRR)—evaluate model performance based on how accurately the retrieved documents match these ground-truth references. In the following, we provide a detailed analysis of the results under these classic IR metrics.

\begin{table}[H]
\vspace{-1em}
\caption{Performance of traditional IR metrics on the BioASQ dataset for the document retrieval task.}
\label{tab:biomed_ir_traditional_doc_retrieval}
\centering
\resizebox{\linewidth}{!}{%
\begin{tabular}{l|ccc|ccc}
\toprule
\textbf{Model} & \textbf{Doc Precision} & \textbf{Doc Recall} &   \textbf{Doc F1} & \textbf{Snip Precision} & \textbf{Snip Recall} & \textbf{Snip F1} \\

\midrule
\ding{172}  \texttt{ChatGPT-4o-chat}        &       16.42\% &    22.22\% &  14.93\% &         4.46\% &      2.77\% &  2.57\% \\
\ding{173} $\texttt{ChatGPT-3.5-turbo-chat}$ &       17.20\% &    20.32\% &  14.36\% &         4.68\% &      2.81\% &  2.72\% \\
\ding{174} $\texttt{LLaMa-V3-70b-instruct}$  &       15.93\% &    20.88\% &  14.08\% &         4.67\% &      \textbf{3.29\%} &  \textbf{2.94\%} \\
\ding{175} $\texttt{LLaMa-V3-3b-instruct}$   &        0.48\% &     1.09\% &   0.56\% &         0.06\% &      0.08\% &  0.06\% \\
\ding{176} $\texttt{Mistral-7B-instruct-V0.1}$ &        9.68\% &    12.69\% &   8.22\% &         2.70\% &      1.54\% &  1.33\% \\
\ding{177} $\texttt{Qwen2.5-7B-instruct}$    &       11.68\% &    13.55\% &   9.21\% &         3.23\% &      1.29\% &  1.46\% \\
\ding{178} $\texttt{Claude-3.7-Sonnet}$      &       15.82\% &    \textbf{23.99\%} &  15.91\% &         4.15\% &      2.68\% &  2.73\% \\
\ding{179} $\texttt{DeepSeek-V3}$            &       \textbf{19.40\%} &    23.03\% &  \textbf{16.76\%} &         \textbf{4.91\%} &      2.76\% &  2.76\% \\
\ding{180} $\texttt{Gemini-1.5-flash}$       &       16.10\% &    21.65\% &  14.36\% &         4.40\% &      2.73\% &  2.62\% \\
\bottomrule
\end{tabular}
}
\vspace{-1em}
\end{table}

\paragraph{Subtask 1: Document Retrieval.} Table~\ref{tab:biomed_ir_traditional_doc_retrieval} presents classic IR metrics (precision, recall, F1) at the document and snippet levels.
According to traditional IR metrics, \texttt{DeepSeek-V3} achieves the highest document-level F1 score (16.76\%), followed closely by \texttt{Claude-3.7-Sonnet} (15.91\%), \texttt{ChatGPT-4o-chat} (14.93\%), and \texttt{ChatGPT-3.5-turbo-chat} (14.36\%). Snippet-level metrics are uniformly low across all models, with \texttt{LLaMA-V3-70B-instruct} (2.94\%) and \texttt{DeepSeek-V3} (2.76\%) performing best. These results suggest that while LLMs are somewhat effective at retrieving relevant documents, they struggle with extracting precise snippet-level answers. Notably, while \texttt{DeepSeek-V3} ranks highest under traditional metrics, it also has the highest estimated number of \emph{unseen} MeSH keywords in our pipeline—demonstrating strong potential for retrieving latent biomedical knowledge. However, if we account for both observed and estimated unseen knowledge, \texttt{ChatGPT-3.5-turbo-chat} emerges as the best performer, with the highest total knowledge count ($N_{\mathrm{tot}} = 10367$ in Table \ref{tab:info-and-diversity}). This ranking shift illustrates how our estimation pipeline reveals model capabilities not captured by conventional evaluation—highlighting output diversity and latent retrieval capacity rather than solely surface-level accuracy.

\begin{table}[H]
\vspace{-1em}
\caption{Performance of traditional IR metrics on the BioASQ dataset for the QA task.}
\label{tab:biomed_ir_traditional_summarization}
\centering
\resizebox{\linewidth}{!}{%
\begin{tabular}{l|c|c|cc|ccc}
\toprule
\textbf{Model} &    \textbf{ROUGE} & \textbf{Yes/No Acc.} & \textbf{Factoid Strict Acc.} & \textbf{Factoid Lenient Acc.} &  \textbf{List F1} & \textbf{List Prec.} & \textbf{List Rec.} \\
\midrule
\ding{172} $\texttt{ChatGPT-4o-chat}$          &  17.41\% &   83.33\% &             3.53\% &              4.71\% &  \textbf{31.42\%} &   \textbf{36.61\%} &  \textbf{30.53\%} \\
\ding{173} $\texttt{ChatGPT-3.5-turbo-chat}$   &  \textbf{20.15\%} &   88.24\% &             1.18\% &              1.18\% &   7.02\% &    8.32\% &   7.02\% \\
\ding{174} $\texttt{LLaMa-V3-70b-instruct}$    &  19.13\% &   91.18\% &             3.53\% &              3.53\% &   1.70\% &    1.92\% &   1.65\% \\
\ding{175} $\texttt{LLaMa-V3-3b-instruct}$     &   8.21\% &   88.24\% &             3.53\% &              3.53\% &   4.04\% &    4.92\% &   3.99\% \\
\ding{176} $\texttt{Mistral-7B-instruct-V0.1}$ &  14.63\% &   27.45\% &             0.00\% &              0.00\% &   4.04\% &    4.83\% &   3.72\% \\
\ding{177} $\texttt{Qwen2.5-7B-instruct}$      &  11.95\% &   87.25\% &             1.18\% &              1.18\% &   0.57\% &    0.52\% &   0.65\% \\
\ding{178} $\texttt{Claude-3.7-Sonnet}$        &  16.86\% &   85.29\% &             \textbf{8.24\%} &             \textbf{10.59\%} &  16.67\% &   17.99\% &  16.46\% \\
\ding{179} $\texttt{DeepSeek-V3}$              &  19.42\% &   \textbf{92.16\%} &             1.18\% &              1.18\% &  18.58\% &   22.52\% &  18.34\% \\
\ding{180} $\texttt{Gemini-1.5-flash}$         &  17.03\% &   \textbf{92.16\%} &             0.00\% &              0.00\% &  18.40\% &   21.89\% &  18.65\% \\
\bottomrule
\end{tabular}
}
\vspace{-1em}
\end{table}

\paragraph{Subtask 2: Question Answering (QA).} 

Table~\ref{tab:biomed_ir_traditional_summarization} reports standard evaluation metrics for biomedical question answering, including ROUGE, accuracy on yes/no and factoid questions, and F1 for list-type answers. \texttt{DeepSeek-V3}, \texttt{Gemini-1.5-flash}, and \texttt{Claude-3.7-Sonnet} all demonstrate strong performance across several metrics. For example, \texttt{DeepSeek-V3} and \texttt{Gemini-1.5-flash} achieve high yes/no accuracy (92.16\%) and list F1 scores above 18\%, while \texttt{Claude-3.7-Sonnet} achieves the best factoid performance (10.59\% lenient accuracy) and strong list-level precision. These results reflect the models' ability to generate correct and structured answers when evaluated against expert annotations. However, our pipeline uncovers a different aspect of model capability. While \texttt{ChatGPT-4o-chat} only ranks moderately under traditional metrics, it demonstrates the highest estimated total knowledge ($N_{\mathrm{tot}} = 19965$), followed by \texttt{DeepSeek-V3} and \texttt{Claude-3.7-Sonnet} (Table~\ref{tab:info-and-diversity}). This suggests that \texttt{ChatGPT-4o-chat} may encode substantially more answer-relevant MeSH knowledge than is directly observed—revealing latent retrieval potential beyond what conventional metrics capture. Once again, our estimator provides a more complete picture of what LLMs could retrieve, even when their immediate outputs fall short under surface-level evaluation.

\end{document}